\documentclass{article}

\usepackage[preprint,nonatbib]{neurips}
\usepackage{mathptmx}
\usepackage{amsmath,amssymb,amsthm,mathtools}
\usepackage{stmaryrd}
\usepackage{tikz-cd}
\usepackage{booktabs}
\usepackage{makecell}   
\usepackage{array}      
\usepackage{enumitem}
\usepackage[dvipsnames]{xcolor}
\usepackage{graphicx}
\usepackage{caption}
\usepackage{listings}
\newsavebox{\stepcodebox}
\usepackage{csquotes}
\usepackage{microtype}
\usepackage[round,authoryear]{natbib}
\usepackage[bookmarks=true]{hyperref}
\usepackage[capitalise]{cleveref}

\hypersetup{
  colorlinks=true,
  linkcolor=blue!50!black,
  citecolor=blue!50!black,
  urlcolor=blue!50!black
}

\graphicspath{{figures/}}
\input{macros/tikzfig.sty}

\tikzstyle{dot}=[inner sep=1pt,minimum width=2mm,minimum height=2mm,draw, shape=circle]
\tikzstyle{bdot}=[inner sep=1pt,minimum width=0.5mm,minimum height=0.5mm,draw=blue,shape=circle,fill=blue, line width=0.6pt]
\tikzstyle{rdot}=[inner sep=1pt,minimum width=0.5mm,minimum height=0.5mm,draw=red,shape=circle,fill=magenta, line width=0.6pt]
\tikzstyle{ydot}=[inner sep=1pt,minimum width=0.5mm,minimum height=0.5mm,draw=orange,shape=circle,fill=yellow, line width=0.6pt]
\tikzstyle{gdot}=[inner sep=1pt,minimum width=0.5mm,minimum height=0.5mm,draw=gray, shape=circle, fill=gray]
\tikzstyle{graydot}=[inner sep=1pt,minimum width=0.5mm,minimum height=0.5mm,draw=gray, shape=circle]
\tikzstyle{cyandot}=[inner sep=1pt,minimum width=0.5mm,minimum height=0.5mm,draw=cyan,shape=circle,fill=cyan, line width=0.6pt]
\tikzstyle{wdot}=[inner sep=1pt,minimum width=2mm,minimum height=2mm,draw, shape=circle, fill=white]
\tikzstyle{Bdot}=[inner sep=1pt,minimum width=2mm,minimum height=2mm,draw, shape=circle, fill=black]

\tikzstyle{K}=[-, line width=1pt]
\tikzstyle{H}=[-, style=dashed]
\tikzstyle{D}=[-, style=dotted]
\tikzstyle{W}=[-, draw=white, fill=white]
\tikzstyle{nB}=[-, draw=blue, line width=0.6pt]
\tikzstyle{nO}=[-, draw=orange, line width=0.6pt]
\tikzstyle{nR}=[-, draw=red, line width=0.6pt]
\tikzstyle{ng}=[-, draw=gray]
\tikzstyle{ngg}=[-, draw=gray, line width=0.6pt]
\tikzstyle{nC}=[-, draw=cyan, line width=0.6pt]
\tikzstyle{nM}=[-, draw=magenta, line width=0.6pt]
\tikzstyle{nY}=[-, draw=yellow, line width=0.6pt]
\tikzstyle{fB}=[-, draw=cyan, fill=cyan, fill opacity=0.25, opacity=0.25]
\tikzstyle{fO}=[-, draw=yellow, fill=yellow, fill opacity=0.35, opacity=0.35]
\tikzstyle{fR}=[-, draw=magenta, fill=magenta, fill opacity=0.25, opacity=0.25]
\tikzstyle{fg}=[-, draw=gray, fill=gray, fill opacity=0.3, opacity=0.3]
\tikzstyle{fBB}=[-, draw=blue, fill=cyan, fill opacity=0.3]
\tikzstyle{fOO}=[-, draw=orange, fill=yellow, fill opacity=0.3]
\tikzstyle{fRR}=[-, draw=red, fill=magenta, fill opacity=0.3]
\tikzstyle{fgg}=[-, draw=gray, fill=gray, fill opacity=0.3]
\tikzstyle{HBB}=[-, style=dashed, draw=blue, fill=cyan, fill opacity=0.3]
\tikzstyle{HOO}=[-, style=dashed, draw=orange, fill=yellow, fill opacity=0.3]
\tikzstyle{HRR}=[-, style=dashed, draw=red, fill=magenta, fill opacity=0.3]
\tikzstyle{Hgg}=[-, style=dashed, draw=gray, fill=gray, fill opacity=0.3]
\tikzstyle{HB}=[-, draw=blue, style=dashed, line width=0.6pt]
\tikzstyle{HO}=[-, draw=orange, style=dashed, line width=0.6pt]
\tikzstyle{HR}=[-, draw=red, style=dashed, line width=0.6pt]
\tikzstyle{Hg}=[-, draw=gray, style=dashed]

\tikzstyle{hardfillL}=[-, draw=black, fill=lime]
\tikzstyle{hardfillY}=[-, draw=black, fill=yellow]
\tikzstyle{hardfillO}=[-, draw=black, fill=orange]
\tikzstyle{dB}=[-, draw=blue, style=densely dotted, line width=0.6pt]
\tikzstyle{dO}=[-, draw=orange, style=densely dotted, line width=0.6pt]
\tikzstyle{dC}=[-, draw=cyan, style=densely dotted, line width=0.6pt]
\tikzstyle{dR}=[-, draw=red, style=densely dotted, line width=0.6pt]
\tikzstyle{dG}=[-, draw=ForestGreen, style=densely dotted, line width=0.6pt]

\tikzstyle{v}=[->]

\tikzstyle{softline}=[-, opacity=0.3]

\tikzstyle{greendot}=[inner sep=1pt,minimum width=0.5mm,minimum height=0.5mm,draw=teal,shape=circle,fill=green, line width=0.6pt]
\tikzstyle{purpledot}=[inner sep=1pt,minimum width=0.5mm,minimum height=0.5mm,draw=violet,shape=circle,fill=blue, line width=0.6pt]
\tikzstyle{browndot}=[inner sep=1pt,minimum width=0.5mm,minimum height=0.5mm,draw=orange,shape=circle,fill=orange, line width=0.6pt]
\tikzstyle{tealdot}=[inner sep=1pt,minimum width=0.5mm,minimum height=0.5mm,draw=ForestGreen,shape=circle,fill=ForestGreen, line width=0.6pt]

\tikzstyle{nGreen}=[-, draw=teal, line width=0.6pt]
\tikzstyle{nPurple}=[-, draw=violet, line width=0.6pt]
\tikzstyle{nBrown}=[-, draw=orange, line width=0.6pt]
\tikzstyle{nTeal}=[-, draw=ForestGreen, line width=0.6pt]

\tikzstyle{fGreen}=[-, draw=green, fill=green, fill opacity=0.3, opacity=0.3]
\tikzstyle{fPurple}=[-, draw=violet, fill=blue, fill opacity=0.3, opacity=0.3]
\tikzstyle{fBrown}=[-, draw=orange, fill=orange, fill opacity=0.3, opacity=0.3]
\tikzstyle{fTeal}=[-, draw=ForestGreen, fill=ForestGreen, fill opacity=0.7, opacity=0.7]

\tikzstyle{dbdot}=[inner sep=1pt,minimum width=0.5mm,minimum height=0.5mm,draw=blue,shape=circle,fill=cyan, line width=0.6pt]
\tikzstyle{drdot}=[inner sep=1pt,minimum width=0.5mm,minimum height=0.5mm,draw=red,shape=circle,fill=magenta, line width=0.6pt]
\tikzstyle{dydot}=[inner sep=1pt,minimum width=0.5mm,minimum height=0.5mm,draw=orange,shape=circle,fill=yellow, line width=0.6pt]
\tikzstyle{dgdot}=[inner sep=1pt,minimum width=0.5mm,minimum height=0.5mm,draw=gray, shape=circle, fill=gray]

\tikzstyle{dgreendot}=[inner sep=1pt,minimum width=0.5mm,minimum height=0.5mm,draw=teal,shape=circle,fill=green, line width=0.6pt]
\tikzstyle{dpurpledot}=[inner sep=1pt,minimum width=0.5mm,minimum height=0.5mm,draw=violet,shape=circle,fill=blue, line width=0.6pt]
\tikzstyle{dbrowndot}=[inner sep=1pt,minimum width=0.5mm,minimum height=0.5mm,draw=orange,shape=circle,fill=orange, line width=0.6pt]
\tikzstyle{dtealdot}=[inner sep=1pt,minimum width=0.5mm,minimum height=0.5mm,draw=Aquamarine,shape=circle,fill=Aquamarine, line width=0.6pt]

\tikzstyle{dorangedot}=[inner sep=1pt,minimum width=0.5mm,minimum height=0.5mm,draw=Maroon,shape=circle,fill=Maroon, line width=0.6pt]
\tikzstyle{dmaroondot}=[inner sep=1pt,minimum width=0.5mm,minimum height=0.5mm,draw=Maroon,shape=circle,fill=Maroon, line width=0.6pt]

\tikzstyle{ndB}=[-, draw=blue, line width=0.6pt]
\tikzstyle{ndO}=[-, draw=orange, line width=0.6pt]
\tikzstyle{ndR}=[-, draw=red, line width=0.6pt]
\tikzstyle{ndg}=[-, draw=gray]

\tikzstyle{ndGreen}=[-, draw=teal, line width=0.6pt]
\tikzstyle{ndPurple}=[-, draw=violet, line width=0.6pt]
\tikzstyle{ndBrown}=[-, draw=orange, line width=0.6pt]
\tikzstyle{ndTeal}=[-, draw=Aquamarine, line width=0.6pt]

\tikzstyle{ndorange}=[-, draw=Maroon, line width=0.6pt]
\tikzstyle{ndmaroon}=[-, draw=Maroon, line width=0.6pt]

\tikzstyle{fdB}=[-, draw=cyan, fill=cyan, fill opacity=0.7, opacity=0.7]
\tikzstyle{fdO}=[-, draw=yellow, fill=yellow, fill opacity=0.8, opacity=0.8]
\tikzstyle{fdR}=[-, draw=magenta, fill=magenta, fill opacity=0.7, opacity=0.7]
\tikzstyle{fdg}=[-, draw=darkgray, fill=darkgray, fill opacity=0.8, opacity=0.8]

\tikzstyle{fdGreen}=[-, draw=green, fill=green, fill opacity=0.7, opacity=0.7]
\tikzstyle{fdPurple}=[-, draw=violet, fill=blue, fill opacity=0.7, opacity=0.7]
\tikzstyle{fdBrown}=[-, draw=orange, fill=orange, fill opacity=0.8, opacity=0.8]
\tikzstyle{fdTeal}=[-, draw=Aquamarine, fill=Aquamarine, fill opacity=0.7, opacity=0.7]

\tikzstyle{fdorange}=[-, draw=RedOrange, fill=RedOrange, fill opacity=0.7, opacity=0.7]
\tikzstyle{fdmaroon}=[-, draw=OrangeRed, fill=OrangeRed, fill opacity=0.7, opacity=0.7]

\tikzstyle{fBlue}=[-, draw=blue, fill=blue, fill opacity=0.7, opacity=0.7]

\newcommand{\sidefigL}[3][0.45]{%
  \par\medskip\noindent%
  \begin{minipage}[c]{#1\linewidth}\centering #2\end{minipage}\hfill%
  \begin{minipage}[c]{\dimexpr 0.97\linewidth - #1\linewidth\relax}#3\end{minipage}%
  \par\medskip%
}
\newcommand{\sidefigR}[3][0.45]{%
  \par\medskip\noindent%
  \begin{minipage}[c]{\dimexpr 0.97\linewidth - #1\linewidth\relax}#2\end{minipage}\hfill%
  \begin{minipage}[c]{#1\linewidth}\centering #3\end{minipage}%
  \par\medskip%
}

\DeclareTextFontCommand{\texttt}{\ttfamily\upshape}

\newcommand{\FinSet}{\mathsf{FinSet}}

\newcommand{\Met}{\mathsf{Met}}

\newcommand{\Disp}{\mathsf{Disp}}
\newcommand{\Euc}{\mathsf{Euc}}
\newcommand{\op}{\mathrm{op}}
\newcommand{\ord}{\mathrm{ord}}
\newcommand{\id}{\mathrm{id}}

\newcommand{\FinPart}{\mathsf{FinPart}}

\theoremstyle{definition}
\newtheorem{definition}{Definition}[section]
\newtheorem{example}[definition]{Example}
\newtheorem{remark}[definition]{Remark}

\theoremstyle{plain}
\newtheorem{proposition}[definition]{Proposition}
\newtheorem{lemma}[definition]{Lemma}
\newtheorem{corollary}[definition]{Corollary}
\newtheorem{theorem}[definition]{Theorem}

\title{Graphical \texttt{einops}:\\
  bridging tensor networks and computation graphs}
\author{%
  Vincent Wang-Ma\'{s}cianica\thanks{Laboratory for Human-Centered AI, Department of Philosophy, University of Oxford. \texttt{vincent.wang-mascianica@philosophy.ox.ac.uk}}
  \And
  Nikhil Khatri\thanks{Machine Learning Research Group, Department of Engineering Science, University of Oxford \texttt{nikhil@robots.ox.ac.uk}}
}
\date{}

\begin{document}
\makeatletter
\renewcommand{\@notice}{}
\makeatother
\setcounter{footnote}{1}
\maketitle
\setcounter{footnote}{0}

\begin{abstract}
Architecture diagrams are ubiquitous in deep learning, but they are usually only representational: the tensor-program identities they suggest are still proved by prose and tensor-axis manipulation. We introduce a formal graphical calculus for the structural fragment of tensor programming underlying \texttt{einops}, making such diagrams proof-enabling. Our calculus represents tensor axes as nested graded tubes around a base type. The tube boundary recovers the undirected tensor-network view of axes, while the directed interior retains the operational reading of computation graphs. The key rewrite is grade-naturality: sliding spectacles over tubes. Standard equivariance proofs become short diagrammatic derivations. We additionally demonstrate how our rewrite system may be applied to convert attention masks into pre-processing operations, recovering efficient implementations of sparse attention blocks.
\end{abstract}

\section{Introduction}
\label{sec:intro}

Tensor-program identities --- ``transpose then reshape equals reshape then transpose'', ``permuting an attention layer's queries permutes its outputs'', ``masked attention equals augmented unmasked attention'' --- are routine in deep learning but proved largely by hand: index manipulation, prose, or appeal to operational intuition. The de facto readable syntax for the structural layer (axis reshapes, broadcasts, contractions, packing) is \texttt{einops}: a procedural surface that disambiguates index choreography but supplies no proof system. We give that layer a graphical proof calculus, where \texttt{einops}-level diagrams are typed proof terms, the rewrite engine is a single naturality square, and structural equality is decidable.

\begin{center}
\begin{tikzpicture}
\node[draw=none, fill=gray, fill opacity=0.1, text opacity=1, rounded corners=4pt, inner sep=10pt, text width=0.94\linewidth] {%
\textbf{Not just notation.}\label{thm:decidable} A well-typed tube diagram denotes a finite read map on tensor positions together with labelled pointwise base maps. In the pointwise SIMD fragment, equality is decidable: slide pointwise boxes outward to normal form, compose the remaining tube-boundary partial functions, and compare on indices. Tube-in-tube-out boxes (contractions, attention) act as landmarks crossable only by declared rules: equivariance and invariance, padding/algebra coherences, augment/pad equivalence. Decidability extends to the fragment plus any finite declared rule set.%
};
\end{tikzpicture}
\end{center}

The reason such a calculus has not appeared earlier is mathematical, not stylistic. Tensor data has an undirected, topological syntax, whereas computation has a directed left-to-right flow. The two views correspond to graphical traditions that disagree on what parallel composition should mean: $\otimes$ or $\oplus$.

The undirected tradition (Penrose tensor diagrams \citep{penrose1971applications}, named-tensor notation \citep{chiang2023named}, and ZX-calculus and its quantum-process descendants \citep{coecke2010interacting}) treats wires as tensor indices and reads diagrams up to free deformation. Pavlovic, Coecke, and Vicary \citep{pavlovic2013new} show that the underlying primitive on every wire type is a special dagger Frobenius algebra (SDFA), and that an SDFA structure on a Hilbert space is equivalent to a choice of orthonormal basis. Among other consequences, parallel composition of diagrams is forced to be the Kronecker product of vector spaces $\otimes$, and the rewrites available bind to that interpretation. Operational structure (concatenation, sequence axes, batching) sits awkwardly here: named-tensor papers explicitly mark pack/unpack-like operations as out of scope, and nonlinear functions are difficult to address formally. Recent DL-flavoured offerings adapt undirected tensor-network diagrams for ML audiences. Taylor's expository diagrams \citep{taylor2024graphical} and Ahle's \texttt{tensorgrad} \citep{ahle_tensorgrad} are notable examples; the latter performs the simplifications of the \emph{Tensor Cookbook} symbolically and lowers diagrams to PyTorch with full higher-order-derivative support. Both treat the diagram chiefly as a frontend for symbolic manipulation and compiler lowering rather than as a proof calculus. At an adjacent abstraction layer, \citet{laue2020simple} give a formal tensor calculus for evaluating and differentiating Einstein-notation contractions: a calculus about derivatives of contractions rather than about the index-and-axis structure of the surrounding tensor program.

The directed tradition (computation graphs, architecture flowcharts) imposes a temporal reading order: operations are boxes, data are wires connecting past outputs to sequentially-composed future inputs. When all such directed data admit copying and deleting maps, Fox's theorem \citep{fox1976coalgebras} forces the meaning of parallel composition to be the categorical product: when we deal with functions between Euclidean spaces in the setting of Deep Learning, parallel composition must mean direct sum. Recent DL-side graphical work in this tradition formalises the layout with explicit SIMD-style boxes for parallel operations along tensor axes: neural circuit diagrams \citep{abbott2024ncd} and the ``Anatomy of Attention'' framework \citep{khatri2024anatomy} are the closest precedents to our calculus and inspired the wires-and-tubes layer used here.

\texttt{einops} \citep{rogozhnikov2022einops} is the practitioner-side bridge between the two views. It is a procedural language for tensor manipulation (\texttt{rearrange}, \texttt{repeat}, \texttt{reduce}, \texttt{pack}, \texttt{unpack}) that sits in the directed tradition, so by Fox its data carries copy and delete. Via the display-map correspondence, each \emph{structural} move determines a finite read map; reductions additionally require a chosen algebra on values. We formalise this bridge as a graphical calculus by wrapping wires in tubes that carry dimension data; nested tubes give multi-axis tensors, partial-function reads on a dense surface handle pack/unpack, and the upper boundary of every tube recovers the undirected tensor-network reading of its axes (\Cref{sec:tubeintro}). The paper's theorem bundle is: (A) structural \texttt{einops} programs embed in a \(\FinSet^{\op}\)-graded tube calculus; (B) equality in the structural fragment reduces to equality of composite functions between finite sets; (C) base morphisms acting pointwise along a grade satisfy a grade-naturality rewrite, \emph{sliding spectacles}; and (D) every attention mask induces a query-indexed display of visible K/V fibres, so masked attention is fibrewise unmasked attention over that display. Causal masking, padding masks, and sliding-window masks are special cases.

\providecommand{\cmrk}{\ensuremath{\checkmark}}
\providecommand{\xmrk}{\ensuremath{\times}}
\providecommand{\pmrk}{\ensuremath{\circ}}
\providecommand{\citelink}[2]{\hyperlink{cite.#2}{#1}}

\begin{table}[!ht]
\caption{\textbf{Calculus comparison.} Property summary of the present calculus against neighbouring graphical, named-tensor, and DSL approaches. \cmrk{} = primary feature; \pmrk{} = supported indirectly or only on a fragment; \xmrk{} = unsupported; ``n/a'' = not applicable to the system's modality. Each system name is a clickable link to its bibliography entry.}
\label{tab:relatedwork}
\centering
\small
\setlength{\tabcolsep}{5pt}
\renewcommand{\arraystretch}{1.2}
\begin{tabular}{@{}lccccccc@{}}
\toprule
\textbf{System} & \makecell{\textbf{Directed}\\\textbf{flow}} & \makecell{\textbf{Axis}\\\textbf{semantics}} & \makecell{\textbf{Copy /}\\\textbf{delete}} & \makecell{\textbf{Pack /}\\\textbf{padding}} & \makecell{\textbf{Proof}\\\textbf{rewrites}} & \makecell{\textbf{Decid.}\\\textbf{struct. eq.}} & \makecell{\textbf{DL}\\\textbf{examples}} \\
\midrule
\citelink{Penrose / TN}{penrose1971applications}
  & \xmrk & free idx. & \pmrk$^a$ & \xmrk & topo+Frob.$^b$ & \pmrk & \pmrk \\
\citelink{Named tensors}{chiang2023named}
  & n/a & named axes & \pmrk & \xmrk$^c$ & \xmrk & \xmrk & \cmrk \\
\citelink{\texttt{tensorgrad}}{ahle_tensorgrad}
  & \xmrk & indexed & \cmrk & \xmrk & symbolic & \pmrk$^d$ & \cmrk \\
\citelink{NCD}{abbott2024ncd}
  & \cmrk & SIMD-named & \cmrk & \xmrk & informal & \xmrk & \cmrk \\
\citelink{Anatomy of Att.}{khatri2024anatomy}
  & \cmrk & SIMD-named & \cmrk & \xmrk & \xmrk & \xmrk & \cmrk \\
\citelink{\texttt{einops}}{rogozhnikov2022einops}
  & n/a$^e$ & patterns & \pmrk & \cmrk & \xmrk & \xmrk & \cmrk \\
\midrule
\textbf{This paper}
  & \cmrk & $\FinSet^{\op}$-graded & \cmrk & \cmrk & grade-nat. & \cmrk$^f$ & \cmrk \\
\bottomrule
\end{tabular}
\\[0.4em]
\raggedright
{\scriptsize $^a$Special-Frobenius copy on the SDFA fragment, distinct from cartesian copy. \quad $^b$Topological deformation plus Frobenius/special-Frobenius rewrites. \quad $^c$Marked out-of-scope by the named-tensor authors. \quad $^d$Symbolic simplification within the tensor-cookbook fragment, not a global decision procedure. \quad $^e$Procedural DSL, not a graphical formalism. \quad $^f$Decidable on the structural fragment (\Cref{thm:decidable}); reductions and softmax interact with declared base algebras and lie outside this guarantee.}
\end{table}
\Cref{tab:relatedwork} situates our work among neighbouring approaches. Unlike prior diagram systems for machine learning, which primarily serve as notation, exposition, or compiler frontends, this calculus provides a formal equality theory for the einops-level structural fragment, together with a normalisation procedure.

\section{Background: building the calculus}
\label{sec:background}

\sidefigR[0.3]{The directed flowchart tradition reads a diagram temporally: operations are boxes, data are wires, and a diagram drawn left-to-right is read in evaluation order. Informal flowcharts in DL papers often differ up to a Poincar{\'e} dual (treating data as boxes and processes as wires) or mix conventions altogether. The string-diagram convention is the formally sound version: Joyal and Street \citep{joyalstreet1991} prove that the desired topological manipulations  (sliding boxes along wires, bending wires, redrawing without changing meaning) agree exactly with the algebraic semantics of symmetric monoidal categories. The interchange law (right) is the headline example: the four parallel-and-sequential arrangements of two boxes $f, g, h, k$ all compute the same morphism. This section walks pedagogically from string-diagram primitives to a calculus that captures \texttt{einops}. The path: a wire denotes $\mathbb{R}$; copy, delete, swap generate every (partial) function on indices via the display-map correspondence; tubes wrap parallel wires; nesting handles multi-axis tensors; \texttt{reshape}, \texttt{rearrange}, \texttt{cat}, \texttt{split} are index bijections; \texttt{pack}/\texttt{unpack} reduces to padding; \texttt{repeat}, \texttt{reduce}, dot product, and \texttt{einsum} follow. The summary table is \Cref{tab:primer}.}{\scalebox{1}{\tikzfig{stringdiagrams/interchange2}}}

\subsection{From maps to diagrams}
\label{sec:displaymaps}

We start with a single wire type. A black wire denotes the field $\mathbb{R}$; $n$ wires in parallel denote $\mathbb{R}^n$. Three structural maps recur in tensor work: \textbf{swap} ($\mathbb{R}^2\to\mathbb{R}^2$, $(a,b)\mapsto(b,a)$, transpose two indices), \textbf{copy} ($\mathbb{R}\to\mathbb{R}^2$, $a\mapsto(a,a)$, broadcast), and \textbf{delete} ($\mathbb{R}\to\mathbf{1}$, $a\mapsto\star$, discard).

\sidefigL{\scalebox{1}{\tikzfig{stringdiagrams/comonoidlaws}}}{The diagrammatic forms of the comonoid axioms say that copy is coassociative (order-independent), cocommutative (branch-order-independent), and counital (unital with respect to delete). All functions are naturally copyable and deletable. These axioms are the entire algebraic content of working with copy and delete, and they are exactly the topological identities that you would intuitively expect of a fork-and-discard structure. The choice to interpret parallel composition as direct sum is what gives us copy and delete in the first place: Fox's theorem \citep{fox1976coalgebras} shows that a symmetric monoidal category in which the tensor product is the categorical product is exactly one in which every object carries a natural comonoid (copy, delete) compatible with the structural maps.}

\sidefigR[0.4]{\textbf{Display-map correspondence.} Tag each input wire with an index, follow the diagram, and read off where each output index lands, going \emph{backwards}. Swap alone gives bijections; adding delete gives injections; adding copy gives every total function between finite sets. Allowing a wire to be \emph{padded} (some output positions have no source and are filled by an external value) extends the correspondence to every partial function. We return to padding once ragged data is on the table.}{\scalebox{1}{\tikzfig{displaymaps/comonoidfnsvert}}}

\vspace{-1em}
\subsection{Tubes wrap parallel wires}
\label{sec:tubeintro}

Drawing $\mathbb{R}^{37}$ as 37 parallel wires would be cumbersome. We compress the bundle into a single coloured tube of grade $37$ carrying one wire; the tube means \emph{direct-sum my contents 37 times}. Two helpful conventions: a tube of grade 1 is invisible (it pops in and out of existence freely), and a tube of grade 0 does not appear at all (neither the tube nor its contents).
\[
\scalebox{0.9}{\tikzfig{matmuls/pita}}\quad\scalebox{0.9}{\tikzfig{matmuls/pitamul}}\quad\scalebox{0.9}{\tikzfig{matmuls/SIMD}}
\]
\textbf{(left)} $k$ parallel wires compress into one tube of grade $k$. \textbf{(middle)} A matrix-shaped bundle of operations from $n$ wires to $m$ wires becomes a single tube-typed box. \textbf{(right)} A base operation applied independently to every wire in a bundle becomes one box \emph{inside} the tube: this is the SIMD pattern.

\textbf{Highlighting the upper boundary.} We will keep the upper boundary of every tube visible. The reason is that highlighting just the upper boundary recovers the undirected tensor-network drawing of the same axes: data moves along wires inside the tube (operational reading), while the boundary records only the indexed type with no temporal ordering (tensor-network reading). \texttt{einops} chooses a directional presentation that coheres with the operational view, but the undirected tensor-network reading remains visible in the same picture as the tube boundary. The two traditions of \Cref{sec:intro} therefore live alongside each other in a single drawing, sharing geometry but not algebraic structure.

\sidefigR[0.25]{\textbf{Nested tubes for multi-axis tensors.} Iterate the wrapping. A tensor $\mathbf{X}$ of shape $[m]\times[n]\times[p]$ becomes an $\textcolor{orange}{m}$-tube around an $\textcolor{blue}{n}$-tube around a $\textcolor{red}{p}$-tube; each axis is one layer of nesting, with the outermost tube as the topmost iterator. Tube colours read inside-out as the axis hierarchy.}{\scalebox{1.5}{\tikzfig{tensormnp}}}

\vspace{-1em}
\subsection{\texttt{Reshape}, \texttt{reorder}, \texttt{cat}, \texttt{split}}
\label{sec:reshape}

We are now in \texttt{einops} territory. Four operations are common in practice: \texttt{reshape} and \texttt{reorder}, \texttt{cat} (concatenation), and \texttt{split} (slicing). All four share the property that they are bijections between indices. By the display-map correspondence, every such bijection is realisable as a wire-swapping braid of black $\mathbb{R}$ wires: nothing happens to the data, only the metadata commitment about how the wires are bundled changes.

\sidefigL[0.4]{\scalebox{1}{\tikzfig{matmuls/reorder}}}{\texttt{Reorder} swaps two tensor indices: a generalised transposition. In our calculus this is a braid that swaps the order of two tubes via their boundaries. The data wires inside are unaffected; the boundary tells the reader which axis is which.}

\vspace{-1em}
\sidefigR[0.2]{\texttt{Reshape} can change the number of nesting layers but is otherwise an identity process. The cardinality isomorphism $[mn]\cong[m]\times[n]$ rewires data between a flat $mn$-tube and a nested $[m]$-outer, $[n]$-inner pair: a tubing can split its integer grade among child tubes, or gather child grades into a parent. Reshape acts on a single axis at a time, with the other dimensions matching: the standard axis constraint of \texttt{einops.rearrange}.}{\scalebox{1}{\tikzfig{matmuls/reshapevert}}}

\vspace{-1em}
\sidefigL{\scalebox{1}{\tikzfig{matmuls/catsplit}}}{\texttt{cat} \textbf{and} \texttt{split} are concatenation and slicing along an axis. Categorically, they are ordered-sum operations on the grade, not tensor contractions on the data. Given segment metadata, cat embeds two consecutive index blocks into one longer ordered block, and split recovers the two summands. Diagrammatically these are \emph{pants} (cat) and \emph{copants} (split). The pants/copants pair is a strict inverse only relative to the retained segment lengths.}

\vspace{-1em}
\subsection{Repeat, reduce, dot product, and matrix multiplication}
\label{sec:contentful}

We can now move to operations that touch the data, building up to \texttt{einsum} in steps.

\vspace{-1em}
\sidefigR[0.3]{\texttt{repeat} \textbf{is a sock.} The simplest content-bearing move is \texttt{einops.repeat}: copy a (possibly tubed) wiring a number of times. Diagrammatically this is a sock that wraps the input in an outer tube of grade equal to the repeat count. On the index side, repeat corresponds to \emph{delete}: the index map for a repeat operation forgets the new outer grade, since every output position reads the same input cell.}{\scalebox{1}{\tikzfig{matmuls/copytube}}}

\vspace{-2em}
\sidefigL[0.3]{\scalebox{1}{\tikzfig{matmuls/reduce}}}{\texttt{reduce} \textbf{is an anti-sock.} \texttt{einops.reduce} is the converse of repeat: it sheds an outermost tubing layer via an algebra on the base, most commonly $\mathrm{sum}$, $\max$, or $\mathrm{mean}$. The diagrammatic dual of the broadcasting sock is a sealing cap on the outermost tube, which the algebra computes the value of.}

\vspace{-1em}
\sidefigR[0.4]{\textbf{Dot product.} The dot product of two vectors is built from gadgets we already have: zip them together with pants, SIMD elementwise multiply, then reduce with sum. This is the canonical example of a tube ``closing'': two tubes of grade $n$ enter, one wire of grade $1$ exits.}{\scalebox{1}{\tikzfig{ML/towards-matrix-contraction/dot-product}}}

\vspace{-1em}
\sidefigL[0.4]{\scalebox{1}{\tikzfig{ML/towards-matrix-contraction/detailedmatmul}}}{\textbf{Tensor contraction.} A general tensor contraction $\sum_n A_{m,n} B_{k,n}$ is built from the dot-product gadget by tubing-up the non-contracted axes. We use repeat-socks to broadcast each operand to the joint axis layout, pants to pair the two contracted tubes, then a SIMDed dot product. This is \texttt{einsum} pictorially.}

\begin{corollary}[Matrix multiplication]
\label{cor:matmul}
Matrix multiplication $AB^\top$ at shape $m\times k$ is the special case of tensor contraction with two two-axis operands sharing one axis. The repeat-socks collapse to a single broadcast on each side, and the calculus simplifies the contraction to the standard form $C_{i,k} = \sum_n A_{i,n}\, B_{k,n}$.
\end{corollary}

A visual simplification of \Cref{cor:matmul} is provable inside the calculus. Since \texttt{broadcast} (sock) corresponds to index-deletion and \texttt{cat} (pants) to index-copy, counitality simplifies the diagram, exposing which index is contracted. With no nontrivial nesting or direct-summation interaction, we lose nothing by depicting only the topmost boundary and treating data wires as separating guardrails.
\[\tikzfig{ML/towards-matrix-contraction/simplerform}\]

\vspace{-1em}
\subsection{(un)/pack: ragged data via padding}
\label{sec:packunpack}

For ragged tensors, \texttt{einops} provides \texttt{pack} and \texttt{unpack}. We have two presentations: we can \textbf{(left)} decorate tubes with packing metadata (the bracket-label refinement of \Cref{app:metadata}), or we can \textbf{(right)} cast every ragged tensor as a dense tensor with padding elements and recover the original data by partial reads. The latter is closer to practice and keeps tube grades as positive integers, so we adopt it here.
\vspace{-1em}
\[\tikzfig{raggedoptions}\]

\sidefigR{\textbf{Padding extends display maps to partial functions.} A wire that has been padded gives a partial function on indices: each output cell either reads from one input cell, or, if it falls in a padded position, requests an external filler. The filler is not determined by the partial function itself. It belongs to a padding interpretation for the relevant base object or to the operation that consumes the padded wire.}{\scalebox{1}{\tikzfig{displaymaps/pfnsvert}}}

(\texttt{un})\texttt{pad} with an arbitrary $\star$ requires bookkeeping for how the filler propagates through subsequent operations. In practice we choose $-\infty$, $0$, or $1$, since they have neutral or absorbing behaviour under $+$, $\times$, and $\sigma$.

\sidefigL[0.3]{\tikzfig{padadd}}{\textbf{Addition absorbs $0$ paddings.} Because $0$ is the unit of the addition semiring, $0$-padding is visually absorbed by summation algebras.}

\sidefigR[0.3]{\textbf{Multiplication absorbs $1$ paddings.} Because $1$ is the unit of the multiplication semiring, $1$-padding is visually absorbed by products.}{\tikzfig{padmult}}

\sidefigL[0.4]{\tikzfig{padsoftmax}}{\textbf{Softmax transforms -$\infty$ padding into $0$ padding.} Softmax (whether normalising or not) is an $k \in \mathbb{N}$-indexed family of operations $\sigma_k: \mathbb{R}^k \rightarrow \mathbb{R}^k$ with the property that a -$\infty$ logit contributes $0$ output, since $e^{-\infty} = 0$. Hence -$\infty$ paddings can slide through softmax, turning into $0$.}

\begin{proposition}[\texttt{pack}/\texttt{unpack} is a special case of \texttt{pad}/\texttt{unpad}]
\label{prop:packpad}
For ragged components $\{X_i\}$ in any dense bounding shape $B$,
\[
  \texttt{pad}\colon \bigsqcup_i X_i \hookrightarrow B,\qquad
  \texttt{unpad}\colon B \rightharpoonup \bigsqcup_i X_i,\qquad
  \texttt{unpad}\circ \texttt{pad} \;=\; \id.
\]
Ordinary \texttt{pack}/\texttt{unpack} is the case $B = [\sum_i |X_i|]$ with contiguous concatenation; non-contiguous variants are captured by the metadata refinement of \Cref{app:metadata}.
\end{proposition}

\begin{proof}[Sketch]
By construction: \texttt{pad} places each valid cell at its embedded coordinate in $B$; \texttt{unpad} is the partial inverse, undefined on padded coordinates.
\end{proof}

\begin{center}
\includegraphics[width=0.65\linewidth]{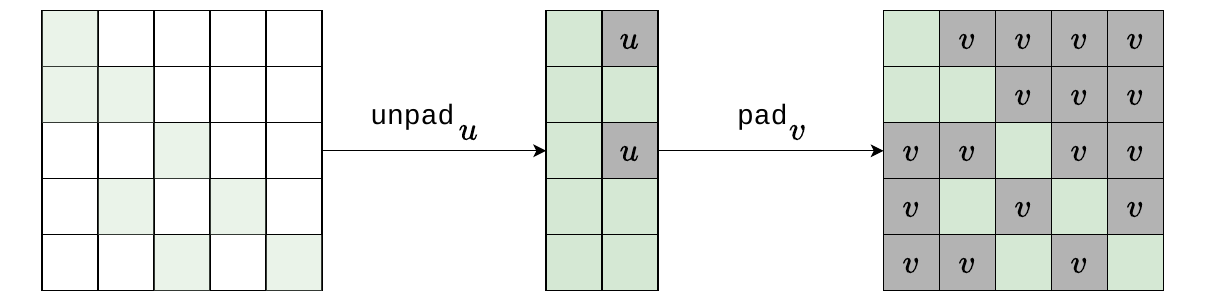}
\captionof{figure}{Action of \texttt{unpad} and \texttt{pad} on a square matrix. Observe that the padding introduced in \texttt{unpad} is overwritten by \texttt{pad}.}
\end{center}

\section{Method: a graded-monad calculus and sliding spectacles}
\label{sec:method}

The formalism behind \Cref{tab:primer} is a graded semantics over the category $\FinSet^{\op}$ of finite sets and (op-of-)functions: full categorical foundations are in \Cref{app:gradedmonad}; the body needs only the consequence below.

\sidefigR[0.3]{%
\begin{proposition}[Grade-naturality, ``sliding spectacles'' (P3)]
\label{prop:slide}
For every base morphism $h\colon A\to B$ and every grade morphism $r\colon Y\to X$ in $\FinSet^{\op}$, the reindexing $r^*$ commutes with $T(h)$:
\[
  T_Y(h) \circ r^* \;=\; r^* \circ T_X(h)\colon\; A^X\to B^Y.
\]
For partial read maps with padding interpretation $\bot$, additionally $h(\bot_A)=\bot_B$.
\end{proposition}
\smallskip
}{\scalebox{1}{\tikzfig{tubediagrams/elemwise}}}
\noindent A grade morphism $r$ acts at both boundaries of a tube, so it is drawn as a \emph{spectacles} bracketing tubing, potentially changing the tubing before-and-after. Naturality says spectacles may slide freely along tubes unless blocked along the tube boundary by a tube-to-tube base morphism: this single rewrite is a proof-engine. When we depict only the top boundary of tubes, \emph{gradings are just string diagrams for partial functions on finite sets.}

\section{Equivariance, in one line per claim}
\label{sec:equivariance}

The statements in this section are deliberately familiar. They calibrate the calculus against facts whose ordinary-index proofs are well-understood: their value here is that each becomes a one-line application of grade-naturality (\Cref{prop:slide}) within the rewrite system of \Cref{tab:primer}. The main stress test of the calculus is the mask-augment theorem of \Cref{sec:maskprep}, where the same rewrite principle controls a less obvious structural transformation.

\subsection{Transformer permutation equivariance}
\label{sec:attnequiv}

Many data transformations have tubed wires as inputs and outputs, and usually they represent boundaries that indexing gadgets on tube boundaries do not interact with. Interesting mathematics occurs when there is interaction. Let $f: \mathbb{R}^k \rightarrow \mathbb{R}^k$ and let $\pi: k \rightarrow k$ denote a bijection; i.e. a permutation of indices, or a \texttt{reorder}.

\sidefigR[0.3]{%
\begin{definition}[Permutation equivariance]
\label{def:permequiv}
A process $f\colon \textcolor{Purple}{A}^{\textcolor{blue}{T}}\to\textcolor{ForestGreen}{B}^{\textcolor{blue}{T}}$ is $\pi$-\emph{equivariant} if it is equal whether $\pi$-\texttt{reorder} occurs before or after, as drawn opposite. Visually, $\pi$ ``hops over". If this holds for all $\pi$, then $f$ is \emph{permutation-equivariant}
\end{definition}
}{\scalebox{1}{\tikzfig{ML/permutation/def-equivariance}}}
\vspace{-1em}

\sidefigL[0.3]{\scalebox{1}{\tikzfig{ML/permutation/def-invariance}}}{%
\begin{definition}[Permutation invariance]
\label{def:perminv}
A process $f\colon \textcolor{ForestGreen}{A}^{\textcolor{blue}{T}}\to \textcolor{Purple}{B}$ is $\pi$-\emph{invariant} if $\pi$-\texttt{reorder} on the inputs has no effect, as drawn opposite. Visually, $\pi$ is ``absorbed" reading one way, and ``freely generated" the other. If this holds for all $\pi$, then $f$ is \emph{permutation invariant}.
\end{definition}
}

\begin{remark}
For $f: \mathbb{R}^j \rightarrow \mathbb{R}^k$ where $j$ and $k$ both support representations of a (finite) group $G$, the definitions above are adaptable to $G$-equivariance and $G$-invariance with the same visual consequences.
\end{remark}

\begin{proposition}[Attention is permutation-equivariant in queries]
\label{prop:q-equiv}
\begin{proof}
$\sigma$ denotes softmax throughout. \textcolor{blue}{Blue} is for \textcolor{blue}{embedding dimension} and \textcolor{orange}{orange} is for \textcolor{orange}{sequence length}. A permutation on the queries' sequence axis slides through to the output.
\[\scalebox{0.75}{\tikzfig{ML/equi0}}\;{=}\;
\scalebox{0.75}{\tikzfig{ML/equi1}}\;{=}\;
\scalebox{0.75}{\tikzfig{ML/equi2}}.\]
\end{proof}
\end{proposition}

\begin{proposition}[Attention is invariant under simultaneous key/value permutation]
\label{prop:kv-inv}
\begin{proof}
A simultaneous permutation on the shared K,V sequence axis is absorbed by the contraction. Reading bottom-up; \textbf{(a)}: $\sigma$ is perm.-equi.; \textbf{(b)}: perms. have inverses; \textbf{(c)}: pants copy/merge indexing; \textbf{(d)}: $\Sigma$ is perm.-inv.
\[\scalebox{0.5}{\tikzfig{ML/inv0vert}}\;\raisebox{0.75cm}{$\overset{}{=}$}\;
\scalebox{0.5}{\tikzfig{ML/inv1vert}}\;\raisebox{0.75cm}{$\overset{(a)}{=}$}\;\scalebox{0.5}{\tikzfig{ML/inv2vert}}\;\raisebox{0.75cm}{$\overset{(b)}{=}$}\;
\scalebox{0.5}{\tikzfig{ML/inv3vert}}\;\raisebox{0.75cm}{$\overset{(c)}{=}$}\;\scalebox{0.5}{\tikzfig{ML/inv4vert}}\;\raisebox{0.75cm}{$\overset{(b)}{=}$}\;
\scalebox{0.5}{\tikzfig{ML/inv5vert}}\;\raisebox{0.75cm}{$\overset{(d)}{=}$}\;\scalebox{0.5}{\tikzfig{ML/inv6vert}}\]
\end{proof}
\end{proposition}

\section{Mask-augment duality}
\label{sec:maskprep}

The folklore explanation of causal masking and positional encoding is that they break the permutation equivariance of attention. \citet{haviv2022transformer} demonstrated empirically that causal masking is sufficient for this purpose, confirming the theoretical observation of \citet{tsai2019transformer}. A na\"{i}ve method to achieve the same goal for an unmasked attention is to augment the data going through the Key and Value streams: take the original sequence $\langle a, b, c, d, ...\rangle$, and create new ``chunked" tokens corresponding to prefixes $\langle [a], [a,b], [a,b,c], [a,b,c,d], ... \rangle$; this is evidently not permutation equivariant, and intuitively there is some shape-level relationship to causal masking. This construction is in fact equal to causal masking, as a corollary of the following general theorem:

\begin{theorem}[Mask-Augment Duality]
\label{thm:duality}
For mask $\mu\colon k \to \mathcal{P}(k)$ encoding the visible positions per query, let $E_\mu = \{(i,j)\in k\times k : j \in \mu(i)\}$ with projections $p,s\colon E_\mu \to k$, $p(i,j) = i$, $s(i,j) = j$. Masked attention equals fibrewise unmasked attention over $p^{-1}(i)$ using $K\circ s$ and $V\circ s$:
\[
  \mathrm{Attn}_\mu(Q,K,V)_i
  \;=\;
  \sum_{j\in\mu(i)} \frac{\exp(\langle Q_i, K_j\rangle / \sqrt{d})}{\sum_{\ell\in\mu(i)} \exp(\langle Q_i, K_\ell\rangle / \sqrt{d})}\, V_j .
\]
Equivalently, the K/V streams are augmented along the tube-function $\mathcal{A}\colon k \twoheadleftarrow \bigsqcup_{i \leq k}\mu(i)$ with $\mathcal{A}^{-1}(i) = \mu(i)$. Rows with $\mu(i)=\varnothing$ are fixed to a default, e.g.\ the zero vector.
\end{theorem}

\begin{proof}
Full lemma statements and a more detailed re-display of the chain are in \Cref{app:maskprepdetail}; an \texttt{einops}-code transcription accompanies it.

\vspace{-1em}
\begin{center}
\setlength{\tabcolsep}{1pt}
\begin{tabular}{@{}ccccccccc@{}}
\scalebox{0.55}{\tikzfig{maskprep/augment0vert}} &
$\rightarrow$ &
\scalebox{0.55}{\tikzfig{maskprep/augment1vert}} &
$\rightarrow$ &
\scalebox{0.55}{\tikzfig{maskprep/augment2vert}} &
$\rightarrow$ &
\scalebox{0.55}{\tikzfig{maskprep/augment3vert}} &
$\rightarrow$ &
\scalebox{0.55}{\tikzfig{maskprep/augment4vert}}
\end{tabular}\\[-3em]
\begin{tabular}{@{}ccccccccccc@{}}
\scalebox{0.55}{\tikzfig{maskprep/augment5vert}} &
$\rightarrow$ &
\scalebox{0.55}{\tikzfig{maskprep/augment6vert}} &
$\rightarrow$ &
\scalebox{0.55}{\tikzfig{maskprep/augment7vert}} &
$\rightarrow$ &
\scalebox{0.55}{\tikzfig{maskprep/augment8vert}} &
$\rightarrow$ &
\scalebox{0.55}{\tikzfig{maskprep/augment9vert}} &
$\rightarrow$ &
\scalebox{0.55}{\tikzfig{maskprep/augment10vert}}
\end{tabular}
\end{center}
\vspace{-2em}
\end{proof}

Three threads of attention-time optimisations are special cases of \Cref{thm:duality}. Attribution, speedup formulae, and on-device validation are in \Cref{app:experiments}; \Cref{fig:summary} (Apple M4 Max) reports the headline numbers.

\begin{corollary}[Hoisting common fibres]
\label{cor:hoisting}
A common sub-fibre factors out of the per-query computation. The strict-left-prefix case is prompt-cache prefix sharing \citep{kwon2023vllm,zheng2024sglang}; the DAG case is shared-context decoding \citep{juravsky2024hydragen,yao2023tree,leviathan2023speculative}.
\end{corollary}

\begin{corollary}[Compacting disjoint fibres]
\label{cor:compaction}
Pairwise-disjoint fibres compact into dense slabs; one SDPA per slab is exact. This is sequence packing \citep{krell2021packing,kosec2021packing} and the variable-length path of FlashAttention, xFormers, and FlexAttention \citep{dao2023flashattention2,xformers,flexattention}.
\end{corollary}

\begin{corollary}[Bounded-component scheduling]
\label{cor:scheduling}
If each row sees at most $c$ intervals of total width $w$, the mask compiles to a row-state vocabulary with $\mathcal{O}(c\,T\,w/B^2)$ tile visits versus $\mathcal{O}(T^2/B^2)$ dense. The $c{=}1$ case is sliding-window attention \citep{beltagy2020longformer,zaheer2020bigbird,jiang2023mistral,xiao2024streamingllm}; $c{\geq}2$ covers local+global mixtures. Kernel-side compilation: FlashMask and FlexAttention \citep{wang2024flashmask,flexattention}.
\end{corollary}

\section{Discussion and conclusion}
\label{sec:discussion}

We have presented a formal diagrammatic proof calculus for the structural layer of tensor programming. This is the layer manipulated by reshape/reorder/repeat/reduce/pack/unpack/pad patterns in einops-style code, ubiquitous in implementations of machine learning architectures. Prior ML diagram systems are useful as notation, communication devices, or compiler frontends, but they do not provide a formal graphical proof calculus for structural equality.

Our calculus formally bridges undirected tensor network notation, and directed computation graphs. Undirected tensor-network calculi are the right syntax for Kronecker-product structure, contraction topology, and Frobenius-basis reasoning. The structural layer of deep-learning programs is different. Batches, sequence axes, broadcasts, repeats, reshapes, packing, and padding are cartesian phenomena: they are about copying, deleting, and reindexing positions. Tubes make this cartesian structure visible while preserving the operational left-to-right reading of computation graphs. Their upper boundaries retain the familiar tensor-network view of axes, but the interior records the directed computation.

The main rewrite in our diagrams is grade-naturality. Pointwise maps do not care how the surrounding finite index set has been rearranged; hence the ``spectacles'' slide. This simple fact accounts for a surprisingly large class of tensor-program equalities. The examples in the paper are intentionally familiar at first: transformer permutation equivariance and key/value permutation invariance become one-line diagrammatic derivations. The mask-augment duality then shows that the same principle also explains a less obvious equivalence: a mask is a query-indexed display of visible key/value fibres, and masked attention is equal to unmasked attention computed fibrewise over that display.

In the appendix, we present several corollaries of this single rewrite.  These recover known architecture implementation patterns from the literature. These serve as validation targets, demonstrating the breadth of results captured in one succinct derivation.

The accompanying code transcription in torch with einops proves the same theorem using only explicit reshapes, packs, pads, unpads, broadcasts, masks, softmax-support changes, and reductions. It is verbose and difficult to audit locally. The diagrammatic proof compresses this derivation into typed local rewrites, each licensed by grade-naturality or an explicitly declared algebra/padding law. Our diagrams are thus more than a presentation layer, and constitute an ergonomic tool to compress useful proofs.

Several directions for further development and application of our contribution follow naturally. A small normaliser for the structural fragment (WIP) would turn the calculus into a practical checker for einops-style rewrites. Derivative tubes would extend the language to backpropagation while preserving the separation between structural reindexing and base arithmetic. Multi-head attention, multi-query
attention, tiled attention, and FlashAttention-style IO rewrites are natural next tests: each combines finite display structure with a small number of declared algebraic or systems-level invariants. More speculatively, the same wires/tubes split may be useful in mechanistic interpretability, where residual-stream read/write structure is often discussed diagrammatically but rarely given a formal axis semantics.

\bibliographystyle{plainnat}
\bibliography{refs}

@inproceedings{rogozhnikov2022einops,
  author    = {Rogozhnikov, Alex},
  title     = {Einops: Clear and Reliable Tensor Manipulations with Einstein-like Notation},
  booktitle = {International Conference on Learning Representations (ICLR)},
  year      = {2022}
}

@article{penrose1971applications,
  author = {Penrose, Roger},
  title  = {Applications of Negative Dimensional Tensors},
  journal= {Combinatorial Mathematics and its Applications},
  year   = {1971}
}

@misc{chiang2023named,
  author = {Chiang, David and Rush, Alexander M. and Barak, Boaz},
  title  = {Named Tensor Notation},
  year   = {2023},
  note   = {arXiv:2102.13196}
}

@misc{coecke2010interacting,
  author = {Coecke, Bob and Duncan, Ross},
  title  = {Interacting Quantum Observables: Categorical Algebra and Diagrammatics},
  year   = {2011},
  note   = {New Journal of Physics 13:043016}
}

@article{pavlovic2013new,
  author  = {Pavlovic, Du{\v s}ko and Coecke, Bob and Vicary, Jamie},
  title   = {A new description of orthogonal bases},
  journal = {Mathematical Structures in Computer Science},
  volume  = {23},
  number  = {3},
  pages   = {555--567},
  year    = {2013}
}

@article{fox1976coalgebras,
  author  = {Fox, Thomas},
  title   = {Coalgebras and Cartesian categories},
  journal = {Communications in Algebra},
  volume  = {4},
  number  = {7},
  pages   = {665--667},
  year    = {1976}
}

@misc{khatri2024anatomy,
  author = {Khatri, Nikhil and Laakkonen, Tuomas and Liu, Jonathon and Wang-Ma{\'s}cianica, Vincent},
  title  = {On the Anatomy of Attention},
  year   = {2024},
  note   = {arXiv:2407.02423}
}

@article{abbott2024ncd,
  author  = {Abbott, Vincent},
  title   = {Neural Circuit Diagrams: Robust Diagrams for the Communication, Implementation, and Analysis of Deep Learning Architectures},
  journal = {Transactions on Machine Learning Research},
  year    = {2024},
  note    = {arXiv:2402.05424}
}

@article{joyalstreet1991,
  author  = {Joyal, Andr{\'e} and Street, Ross},
  title   = {The geometry of tensor calculus, {I}},
  journal = {Advances in Mathematics},
  volume  = {88},
  number  = {1},
  pages   = {55--112},
  year    = {1991}
}

@misc{taylor2024graphical,
  author = {Taylor, Jordan K.},
  title  = {An introduction to graphical tensor notation for mechanistic interpretability},
  year   = {2024},
  note   = {arXiv:2402.01790}
}

@misc{ahle_tensorgrad,
  author = {Ahle, Thomas D.},
  title  = {The Tensor Cookbook and {\texttt{tensorgrad}}: machine learning with symbolic tensors},
  year   = {2024},
  note   = {Open-source library \url{https://github.com/thomasahle/tensorgrad} and textbook draft \url{https://tensorcookbook.com/}.}
}

@article{laue2020simple,
  author  = {Laue, S{\"o}ren and Mitterreiter, Matthias and Giesen, Joachim},
  title   = {A Simple and Efficient Tensor Calculus for Machine Learning},
  journal = {Fundamenta Informaticae},
  volume  = {177},
  number  = {2},
  pages   = {157--179},
  year    = {2020},
  note    = {Also AAAI 2020; arXiv:2010.03313}
}

@inproceedings{mellies2006functorial,
  author    = {Melli{\`e}s, Paul-Andr{\'e}},
  title     = {Functorial Boxes in String Diagrams},
  booktitle = {Computer Science Logic (CSL)},
  year      = {2006},
  pages     = {1--30},
  note      = {LNCS 4207}
}

@misc{moeller2020tube,
  author = {Moeller, Joe},
  title  = {Tube diagrams for monoidal monads},
  year   = {2020},
  note   = {Blog post, 9 July 2020, \url{https://joe-moeller.com/2020/07/09/tube-diagrams-for-monoidal-monads/}}
}

@inproceedings{haviv2022transformer,
  author    = {Haviv, Adi and Ram, Ori and Press, Ofir and Izsak, Peter and Levy, Omer},
  title     = {Transformer Language Models without Positional Encodings Still Learn Positional Information},
  booktitle = {Findings of the Association for Computational Linguistics (EMNLP)},
  year      = {2022},
  pages     = {1382--1390},
  note      = {arXiv:2203.16634}
}

@inproceedings{tsai2019transformer,
  author    = {Tsai, Yao-Hung Hubert and Bai, Shaojie and Yamada, Makoto and Morency, Louis-Philippe and Salakhutdinov, Ruslan},
  title     = {Transformer Dissection: An Unified Understanding for Transformer's Attention via the Lens of Kernel},
  booktitle = {Empirical Methods in Natural Language Processing (EMNLP-IJCNLP)},
  year      = {2019},
  pages     = {4343--4352}
}

@inproceedings{kwon2023vllm,
  author    = {Kwon, Woosuk and Li, Zhuohan and Zhuang, Siyuan and Sheng, Ying and
               Zheng, Lianmin and Yu, Cody Hao and Gonzalez, Joseph E. and Zhang, Hao and
               Stoica, Ion},
  title     = {Efficient Memory Management for Large Language Model Serving with {PagedAttention}},
  booktitle = {Proceedings of the 29th Symposium on Operating Systems Principles (SOSP)},
  year      = {2023}
}

@inproceedings{zheng2024sglang,
  author    = {Zheng, Lianmin and Yin, Liangsheng and Xie, Zhiqiang and Sun, Chuyue and
               Huang, Jeff and Yu, Cody Hao and Cao, Shiyi and Kozyrakis, Christos and
               Stoica, Ion and Gonzalez, Joseph E. and Barrett, Clark and Sheng, Ying},
  title     = {{SGLang}: Efficient Execution of Structured Language Model Programs},
  booktitle = {Advances in Neural Information Processing Systems (NeurIPS)},
  year      = {2024}
}

@inproceedings{juravsky2024hydragen,
  author    = {Juravsky, Jordan and Brown, Bradley and Ehrlich, Ryan and Fu, Daniel Y. and
               R{\'e}, Christopher and Mirhoseini, Azalia},
  title     = {{Hydragen}: High-Throughput {LLM} Inference with Shared Prefixes},
  booktitle = {International Conference on Machine Learning (ICML)},
  year      = {2024}
}

@inproceedings{yao2023tree,
  author    = {Yao, Shunyu and Yu, Dian and Zhao, Jeffrey and Shafran, Izhak and
               Griffiths, Thomas L. and Cao, Yuan and Narasimhan, Karthik},
  title     = {Tree of Thoughts: Deliberate Problem Solving with Large Language Models},
  booktitle = {Advances in Neural Information Processing Systems (NeurIPS)},
  year      = {2023}
}

@inproceedings{wang2023selfconsistency,
  author    = {Wang, Xuezhi and Wei, Jason and Schuurmans, Dale and Le, Quoc V. and
               Chi, Ed H. and Narang, Sharan and Chowdhery, Aakanksha and Zhou, Denny},
  title     = {Self-Consistency Improves Chain of Thought Reasoning in Language Models},
  booktitle = {International Conference on Learning Representations (ICLR)},
  year      = {2023}
}

@inproceedings{leviathan2023speculative,
  author    = {Leviathan, Yaniv and Kalman, Matan and Matias, Yossi},
  title     = {Fast Inference from Transformers via Speculative Decoding},
  booktitle = {International Conference on Machine Learning (ICML)},
  year      = {2023}
}

@misc{krell2021packing,
  author = {Krell, Mario Michael and Kosec, Matej and Perez, Sergio P. and Fitzgibbon, Andrew},
  title  = {Efficient Sequence Packing without Cross-contamination: Accelerating Large Language Models without Impacting Performance},
  year   = {2021},
  note   = {arXiv:2107.02027}
}

@misc{kosec2021packing,
  author = {Kosec, Matej and Fu, Sheng and Krell, Mario Michael},
  title  = {Packing: Towards 2x {NLP} {BERT} Acceleration},
  year   = {2021},
  note   = {arXiv:2107.02027v3 (companion blog post)}
}

@inproceedings{dao2022flashattention,
  author    = {Dao, Tri and Fu, Daniel Y. and Ermon, Stefano and Rudra, Atri and R{\'e}, Christopher},
  title     = {{FlashAttention}: Fast and Memory-Efficient Exact Attention with {IO}-Awareness},
  booktitle = {Advances in Neural Information Processing Systems (NeurIPS)},
  year      = {2022}
}

@inproceedings{dao2023flashattention2,
  author    = {Dao, Tri},
  title     = {{FlashAttention-2}: Faster Attention with Better Parallelism and Work Partitioning},
  booktitle = {International Conference on Learning Representations (ICLR)},
  year      = {2024},
  note      = {arXiv:2307.08691, 2023}
}

@misc{xformers,
  author = {Lefaudeux, Benjamin and Massa, Francisco and Liskovich, Diana and Xiong, Wenhan and
            Caggiano, Vittorio and Naren, Sean and Xu, Min and Hu, Jieru and Tintore, Marta and
            Zhang, Susan and Labatut, Patrick and Haziza, Daniel and Wehrstedt, Luca and
            Reizenstein, Jeremy and Sizov, Grigory},
  title  = {{xFormers}: A Modular and Hackable Transformer Modelling Library},
  year   = {2022},
  howpublished = {\url{https://github.com/facebookresearch/xformers}}
}

@misc{flexattention,
  author = {Dong, Juhan and Feng, Yanbo and He, Horace and Guessous, Driss and
            Liang, Yanli and Zhang, Joy Dong},
  title  = {{FlexAttention}: The Flexibility of {PyTorch} with the Performance of {FlashAttention}},
  year   = {2024},
  howpublished = {PyTorch Blog, August 2024}
}

@misc{child2019sparsetransformer,
  author = {Child, Rewon and Gray, Scott and Radford, Alec and Sutskever, Ilya},
  title  = {Generating Long Sequences with Sparse Transformers},
  year   = {2019},
  note   = {arXiv:1904.10509}
}

@misc{beltagy2020longformer,
  author = {Beltagy, Iz and Peters, Matthew E. and Cohan, Arman},
  title  = {{Longformer}: The Long-Document Transformer},
  year   = {2020},
  note   = {arXiv:2004.05150}
}

@inproceedings{zaheer2020bigbird,
  author    = {Zaheer, Manzil and Guruganesh, Guru and Dubey, Kumar Avinava and Ainslie, Joshua and
               Alberti, Chris and Onta{\~n}{\'o}n, Santiago and Pham, Philip and Ravula, Anirudh and
               Wang, Qifan and Yang, Li and Ahmed, Amr},
  title     = {Big Bird: Transformers for Longer Sequences},
  booktitle = {Advances in Neural Information Processing Systems (NeurIPS)},
  year      = {2020}
}

@misc{jiang2023mistral,
  author = {Jiang, Albert Q. and Sablayrolles, Alexandre and Mensch, Arthur and Bamford, Chris and
            Chaplot, Devendra Singh and de las Casas, Diego and Bressand, Florian and Lengyel, Gianna and
            Lample, Guillaume and Saulnier, Lucile and Lavaud, L{\'e}lio Renard and Lachaux, Marie-Anne and
            Stock, Pierre and Le Scao, Teven and Lavril, Thibaut and Wang, Thomas and Lacroix, Timoth{\'e}e and
            El Sayed, William},
  title  = {{Mistral 7B}},
  year   = {2023},
  note   = {arXiv:2310.06825}
}

@inproceedings{xiao2024streamingllm,
  author    = {Xiao, Guangxuan and Tian, Yuandong and Chen, Beidi and Han, Song and Lewis, Mike},
  title     = {Efficient Streaming Language Models with Attention Sinks},
  booktitle = {International Conference on Learning Representations (ICLR)},
  year      = {2024}
}

@misc{gemma2024,
  author = {{Gemma Team}, Google DeepMind},
  title  = {{Gemma 3}: Open Models with Local-Global Alternating Attention},
  year   = {2024},
  note   = {Technical report; sliding window 1024, 5:1 sliding/full alternation.}
}

@misc{wang2024flashmask,
  author = {Wang, Guoxia and Wei, Tengda and Lin, Hailang and Lin, Jiacheng and Zhao, Yuang and
            Liu, Junyuan and Wang, Liang and Wang, Dianhai and Liu, Yu and Wang, Haifeng},
  title  = {{FlashMask}: Efficient and Rich Mask Extension of {FlashAttention}},
  year   = {2024},
  note   = {arXiv:2410.01359}
}

@misc{mlxlm,
  author = {{MLX Team}, Apple},
  title  = {{MLX-LM}: Language Model Inference and Training with {MLX}},
  year   = {2024},
  howpublished = {\url{https://github.com/ml-explore/mlx-lm}}
}

@misc{lmstudiocommunity,
  author = {{LM Studio Community}},
  title  = {{LM Studio} community {GGUF} / {MLX} model packs (Hermes~4~70B,
            Gemma~4~31B, Gemma~4~26B-A4B, Llama~3.3~70B-Instruct, GPT-OSS~120B,
            Qwen3-Next 80B-A3B)},
  year   = {2024--2026},
  howpublished = {\url{https://huggingface.co/lmstudio-community}}
}

\appendix
\clearpage

\vspace*{\fill}
\section{Cheat sheet}
\label{app:primer}

\begin{table}[!ht]
\caption{\textbf{The primer (summary cheat sheet).}}
\label{tab:primer}
\centering
\small
\renewcommand{\arraystretch}{1.5}
\begin{tabular}{@{}>{\centering\arraybackslash}m{0.32\linewidth}m{0.30\linewidth}m{0.32\linewidth}@{}}
\toprule
\textbf{Diagram} & \textbf{Categorical content} & \textbf{\texttt{einops} / code reading} \\
\midrule
\scalebox{1}{\tikzfig{tubediagrams/elemwise}}
 & \emph{Spectacles.} Grade-naturality: $h$ slides past any grade. \emph{The rewrite engine.}
 & \texttt{f(rearrange(x, ...)) == rearrange(f(x), ...)}. \\
\midrule
\scalebox{1}{\tikzfig{matmuls/SIMDdisp}}
 & \emph{SIMD.} Apply $f$ pointwise along an axis.
 & \texttt{f(x)}. \\
\midrule
\scalebox{1}{\tikzfig{matmuls/reshapedisp}}
 & \emph{Reshape.} Cardinality iso $[mn]\cong[m]\times[n]$.
 & \texttt{rearrange(x, "(g\,h) c -> g h c")}. \\
\midrule
\scalebox{1}{\tikzfig{matmuls/reorderdisp}}
 & \emph{Braid (reorder).} Permutation on wire order.
 & \texttt{rearrange(x, "b h w c -> b c h w")}. \\
\midrule
\scalebox{1}{\tikzfig{matmuls/copytubedisp}}
 & \emph{Sock.} Copy comonoid: broadcast a new outer grade.
 & \texttt{repeat(x, "d -> b d", b=B)}. \\
\midrule
\scalebox{1}{\tikzfig{matmuls/addtubedisp}}
 & \emph{Anti-sock.} Reduce: many-to-one read $+$ base addition.
 & \texttt{reduce(x, "b d -> d", "sum")}. \\
\midrule
\scalebox{1}{\tikzfig{matmuls/bundledisp}}
 & \emph{Pants.} Cat: ordered concatenation.
 & \texttt{pack([x, y], "* d")}. \\
\midrule
\scalebox{1}{\tikzfig{matmuls/unbundledisp}}
 & \emph{Copants.} Split: ordered slice.
 & \texttt{unpack(z, [shape\_x, shape\_y], "* d")}. \\
\midrule
\scalebox{1}{\tikzfig{ML/towards-matrix-contraction/dot-productdisp}}
 & \emph{Dot product.} Sock $+$ pants $+$ anti-sock: tube closing.
 & \texttt{einsum("d, d ->", v, w)}. \\
\midrule
\scalebox{1}{\tikzfig{ML/towards-matrix-contraction/tensor-contrdisp}}
 & \emph{Matmul.} Dot-product tubed by an axis: $AB^\top$.
 & \texttt{einsum("m\,n, k\,n -> m\,k", A, B)}. \\
\midrule
\multicolumn{3}{c}{%
\begin{minipage}{0.95\linewidth}
\centering
\scalebox{1}{\tikzfig{ML/attnbase}}\\[0.5em]
\emph{Unmasked attention.}\\[0.25em]
\resizebox{\linewidth}{!}{\texttt{einsum("i j, j d -> i d", softmax(einsum("i d, j d -> i j", Q, K) / sqrt(d), dim=-1), V)}}
\end{minipage}%
} \\
\bottomrule
\end{tabular}
\end{table}
\vspace*{\fill}

\clearpage

\section{The graded monad: \texorpdfstring{$\FinSet^{\op}$}{FinSet\textasciicircum op} over \texorpdfstring{$\Euc$}{Euc}}
\label{app:gradedmonad}

The formal algebraic semantics of tubes cohering with string diagrams is a symmetric monoidal graded monad; these proofs assume familiarity with symmetric monoidal categories.

\begin{definition}[Symmetric monoidal graded monad]
\label{def:gradedT}
A graded monad consists of a \emph{grading category} $\mathcal{C}$ over a \emph{base category} $\mathcal{D}$, along with a functor $\mathbb{G}$ from the grading into the endofunctor category of $\mathcal{D}$, $\mathbb{G}: \mathcal{C} \rightarrow [\mathcal{D},\mathcal{D}]$. Endofunctor categories always carry monoidal structure $([\mathcal{D},\mathcal{D}],\circ,\mathsf{id}_{\mathcal{D}})$, where only some pairs of endofunctors $\mathbf{J},\mathbf{K}: \mathcal{D} \rightarrow \mathcal{D}$ may braid symmetric monoidally (in the presence of natural isomorphisms $\mathbf{J} \circ \mathbf{K} \Leftrightarrow \mathbf{K} \circ \mathbf{J}$). When the grading category $\mathcal{C}$ and the functor $\mathbb{G}$ are symmetric monoidal we have a \emph{Symmetric monoidal graded monad}.
\end{definition}

Throughout the body, we have used a $\FinSet^{\op}$ graded-monad over $\Euc$; we will handle the padded fragment $\FinPart^{\op}$ by a reduction to a $\FinSet^{\op}$-grading. $\FinSet^{\op}$ is the opposite of the category of finite sets and partial functions, with cartesian product as symmetric monoidal structure. $\Euc$ is the category of Euclidean spaces and arbitrary functions, with direct sum as symmetric monoidal structure. The connection between the algebraic categorical content and the diagrammatic notation is already established by Melli\`es' Functor Boxes \citeyearpar{mellies2006functorial}, where tubes are an evident extension to the symmetric monoidal case \citep{moeller2020tube}. Hence all we have to elaborate and verify is that we indeed have a symmetric monoidal graded monad.

We proceed in two steps. First, we construct and verify the $(\FinSet,\times,1)^{\op}$ grading over $([\Euc,\Euc],\circ,\id_{\Euc})$. Second, we demonstrate that padding bookkeeping of $\FinPart^{\op}$ is notational: a $\FinSet^{\op}$ grading suffices.

\subsection{\texorpdfstring{$\FinSet^{\op}$}{FinSet\textasciicircum op} graded monad over \texorpdfstring{$\Euc$}{Euc}}

We will proceed by describing the desired functor $\mathbb{G}$, and then we verify its required properties.

\paragraph{The object part of \texorpdfstring{$\mathbb{G}$}{G}}

The functor $\mathbb{G}: (\FinSet,\times,1)^{\op} \rightarrow ([\Euc,\Euc],\circ,\id_{\Euc})$ sends $\FinSet \ni X$ to the endofunctor $(-)^{\oplus |X|}: \Euc \rightarrow \Euc$. Equivalently, in the cartesian-base form most useful for the body's calculus, write $T_X(A) := A^X$ for the indexed-product object. The unit and product cases are then
\[
  T_1(A) \cong A,
  \qquad
  T_{X\times Y}(A) \cong T_X(T_Y(A)),
\]
which is the strong graded structure: the cartesian product on grades corresponds to the composition on endofunctors. Respecting our outside-in reading convention for tubing, the functor sends $X \times Y$ to the composite endofunctor $(-)^{\oplus |X|} \circ (-)^{\oplus |Y|}$. On morphisms, a finite read map $r\colon Y\to X$ in $\FinSet$ acts contravariantly as the reindexing
\[
  r^*\colon A^X \to A^Y,
  \qquad
  (r^* a)_y \;=\; a_{r(y)},
\]
recovering the display-map view of the body: the read map points opposite to data flow, sending each output index to the input index it reads from.

\paragraph{The morphism part of \texorpdfstring{$\mathbb{G}$}{G}}

In this section we must formally justify the display-map argument that relates copy-delete to functions between finite sets. We proceed by the following steps; first, we establish a helper lemma concerning $\Euc$ that renders its objects as positive integers; second, we establish a normal-form factorisation of functions between finite sets (viewable as integers) in terms of bijections, surjections, and injections; third, we establish a one-to-one correspondence between the normal form factorisation and composites of copies, deletes, and braidings in $\Euc$; fourth, we establish that the correspondence holds for tensor products, which are equivalently sequentially composed tupling endomorphisms, or cartesian products of finite sets in the grading category.

First, the intermediary lemma concerning $\Euc$, justifying our convention throughout of only using a single base wire $\mathbb{R}$.

\begin{lemma}[One wire type suffices for $\Euc$]
\label{cor:onewiretype}
Working in the ambient setting of (not necessarily smooth) functions between Euclidean spaces, only one wire type $\mathbb{R}$ is needed.
\begin{proof}
The direct sum is the categorical product of Euclidean spaces, which Fox's theorem licenses as the parallel composition in the presence of natural copy-delete maps. Since every Euclidean space is a direct sum of $\mathbb{R}$, we are done.
\end{proof}
\end{lemma}

Lemma \ref{cor:onewiretype} allows us to deal with just a single copy map $\delta: \mathbb{R} \rightarrow \mathbb{R} \oplus \mathbb{R}$ and a single delete map $\epsilon: \mathbb{R} \rightarrow \{\star\}$; the 0-dimensional space is isomorphic to the singleton set.

Second, the normal form. Where $n$ is a positive integer, let $[n]$ denote the set $\{0,1,...,(n-1)\}$; without loss of generality for the arguments to follow, we may consider finite sets of this form \footnote{We cannot in general \emph{only} consider finite sets of this form (the skeleton of $\FinSet^\op$) as we must also distinguish the data of disjoint unions and cartesian products, which are bookkept by isomorphisms in the ambient $\FinSet$. It is folklore that the free PROP generated by a cocommutative comonoid is precisely this skeleton $(\FinSet,+,0)^\op$; this folklore essentially settles the morphism part of the functor, but we provide a proof anyway.}

\begin{proposition}
\label{prop:functionnormalform}
Every function $f: [n] \rightarrow [m]$ admits a canonical factorisation of two bijections, a surjection, and an injection $[n] \leftrightarrow [n] \twoheadrightarrow [|\mathsf{Im}(f)|] \leftrightarrow [|\mathsf{Im}(f)|] \hookrightarrow [m]$, where the surjective and injective legs are isotone (order-respecting).
\begin{proof}
The surjection and injection arise for free from the usual epi-mono factorisation available in $\FinSet$. The bijections rearrange $[n]$ and $[|\mathsf{Im}(f)|]$ such that the fibres of the surjection and injection respectively are contiguous, which render the surjection and injection isotone whilst preserving composite equality with the original $f$.
\end{proof}
\end{proposition}

Third, the correspondence to copy-deletes.

\begin{proposition}
\label{prop:copydeletenormal}
Up to isomorphism, every morphism $g: \mathbb{R}^m \rightarrow \mathbb{R}^{n}$ that is formed of a composite of $\delta$, $\epsilon$, and braiding corresponds to a unique function $f: [n] \rightarrow [m]$.
\begin{proof}
By counitality, every $g$ is isomorphic to $g'$ where there are no $\epsilon$ sequentially composed with any leg of a $\delta$. By braid naturality and expanding/contracting identity maps we may organise all the $\epsilon$ to occur sequentially before all $\delta$: $g'$ is isomorphic to a braiding, followed by a tensor of the form $\id_{\mathbb{R}^x} \oplus (\epsilon)^y$, followed by a second braiding, followed by a tensored composite of $\delta$ (by coassociativity, the order of $\delta$ is immaterial). Up to an isomorphism in $\FinSet$ recasting coproducts (disjoint unions) as $[n]$, the $\delta$-block and $\epsilon$-block agree with coproduct factorisations of the desired surjection and injection respectively in $\FinSet^\op$ obtained via Proposition \ref{prop:functionnormalform}.
\end{proof}
\end{proposition}

Now we may state ($\times$-free) action of $\mathbb{G}$ on morphisms.

\begin{definition}[$\mathbb{G}$ on $\times$-free morphisms]
$\mathbb{G}$ sends finite functions $f: Y \rightarrow X$ in $\FinSet^\op$ to the function $\mathbb{G}(f): \mathbb{R}^{|X|} \rightarrow \mathbb{R}^{|Y|}$ in $\Euc$ composed of $\delta$, $\epsilon$ and braiding, given by Propositions \ref{prop:functionnormalform} and \ref{prop:copydeletenormal}, which guarantee well-definedness and uniqueness up to isomorphism.
\end{definition}

Finally, we address the behaviour of $\mathbb{G}$ on morphisms with respect to $\times$. For this it suffices to note that pairs of integers can be ordered lexicographically, which establish isomorphisms through which Propositions \ref{prop:functionnormalform} and \ref{prop:copydeletenormal} carry through.

\subsection{Completing the verification}

\begin{proposition}
$\mathbb{G}$ is (strong) symmetric monoidal.
\begin{proof}
We have two tasks; first we must establish the symmetric monoidality of $\mathbb{G}$; second we must verify naturality in the image $\mathbb{G}$. For the first task, the only point of interest is the specification that the braids $\gamma_{A,B}: B \times A \leftarrow A \times B$ of $\FinSet^\op$ are sent to the (\texttt{reorder}) natural isomorphisms $((-)^{|B|})^{|A|} \Leftrightarrow ((-)^{|A|})^{|B|}$ in $[\Euc,\Euc]$; from this the monoidality and Yang-Baxter equations evidently follow. For the second task, it suffices to observe that braids are definitionally natural transformations, and that by Fox's theorem and corollary \ref{cor:onewiretype} all morphisms in $\Euc$ are cohomomorphisms with respect to the $\delta$-$\epsilon$ comonoid.
\end{proof}
\end{proposition}

\begin{corollary}
\label{prop:semirigfunctor}
$\mathbb{G}$ is a semirig functor, additionally sending the coproduct of $\FinSet^\op$ to $(- \ \oplus =)$.
\begin{proof}
The same proof strategy above carries over to the monoidal structures $(\FinSet^\op,+,0)$ and $([\Euc,\Euc],\oplus,\Euc \rightarrow \mathbb{R}^0)$. The verification of the interaction (e.g. distributivity) of the monoidal products is routine.
\end{proof}
\end{corollary}

\subsection{Handling (un)pad}

Concretely in $\Euc$, the following is just the observation that we can explicitly introduce paddings $\star$ via a direct sum of copies of $\star$, at which point we are dealing with tensor-axes manipulation that falls within the $\FinSet^\op$ grading.

\paragraph{Explicit form.} Fix a padding interpretation $\bot_A \in A$ for each base object. A partial read map $r\colon Y\rightharpoonup X$ acts on $A^X$ as
\[
  (r^*_{\bot_A} a)_y \;=\;
  \begin{cases}
    a_{r(y)} & y \in \mathrm{dom}(r),\\
    \bot_A & y \notin \mathrm{dom}(r).
  \end{cases}
\]
This is exactly the dense action $r^*$ of the previous subsection, extended to undefined positions by reading from the chosen filler. Grade-naturality (\Cref{prop:slide}) then has a single compatibility condition: a base morphism $h\colon A\to B$ slides past a partial read $r$ iff
\[
  h(\bot_A) \;=\; \bot_B,
\]
i.e.\ $h$ preserves the chosen padding element. This is the body's ``compatible padding'' condition spelt out at the level of the indexed-product action.

\paragraph{Reduction to \texorpdfstring{$\FinSet^{\op}$}{FinSet\textasciicircum op}.}
Because we parameterise padding with a single element at a time, we can exploit two observations. First, because partial functions arise as the Kleisli category of the Maybe monad $(- + \bot)$, $r\colon A \rightharpoonup B$ are viewable as functions $\hat{r}\colon A \to B + \{\bot\}$, where elements in $A$ outside the defined domain of $r$ are mapped to a distinguished additional formal element $\bot$; the action of $\hat{r}$ on the preimage of $\bot$ corresponds to copying the padding element that may be introduced as an explicit direct-summed morphism in $\Euc$. Second, by Corollary \ref{prop:semirigfunctor}, we can handle coproducts in the grading category $\FinSet^\op$ algebraically, though we have opted in the body not to depict them graphically to keep tubing visually simple. To eliminate partial functions altogether, it suffices to note that Kleisli composition can be made explicit in the grading $\FinSet^\op$, holding aside a formal element $\bot$; although in practice, we do not come across instances in which padded data is subject to arbitrary transformations; usually they are absorbed or bookkept for removal.

\section{Metadata displays}
\label{app:metadata}

\noindent\emph{This appendix is optional reading.} The dense-padding proof of \Cref{thm:duality} (\Cref{sec:maskprep}) does not depend on it: the body uses only the $\FinSet^{\op}$ grading of \Cref{app:gradedmonad} and the dense-with-partial-reads presentation of \Cref{prop:packpad}. Metadata displays supply the alternative grading category used when one wants diagrams to remember the logical block decomposition of a packed buffer (rather than compiling it away into a dense partial read). Readers focused on the body theorems may skip this appendix; readers interested in non-contiguous packing, structured ragged batches, or block-aware rewriting will find the construction here.

\bigskip
\noindent The dense calculus grades by finite sets of tensor positions. For
ragged and packed tensors this is insufficient: two packed buffers may have the same
underlying finite set of cells while carrying different logical decompositions into
blocks. We therefore replace a bare finite set by a finite display
\[
  p:X\to B,
\]
where \(X\) is the finite set of dense cells and \(B\) is the finite set of logical
blocks. The fibre \(p^{-1}(b)\) is the block of cells belonging to \(b\).

The purpose of the construction is threefold. First, we define a category
\(\Disp(\Met)\) whose morphisms preserve this block metadata. Second, we check that
\(\Disp(\Met)\) is symmetric monoidal, so it can be used as a grading category for
tubed diagrams. Third, we relate it back to ordinary $\FinSet^\op$ grading by an adjunction:
forgetting the display recovers the dense set of cells, while the free embedding equips
a finite set with trivial display structure.

\subsection{The metadata category \texorpdfstring{\(\Met\)}{Met}}

Let \([n]=\{0,\ldots,n-1\}\), regarded as a finite total order. We define \(\Met\) to have as objects
the least class of finite posets containing all \([n]\) and closed under
\[
  P,Q\longmapsto P\times Q,
  \qquad
  P,Q\longmapsto P+Q,
  \qquad
  P,Q\longmapsto P+_{\ord}Q.
\]
Here \(P\times Q\) is the cartesian product (lexicographic) order, \(P+Q\) is disjoint union with no
new order relations between the summands, and \(P+_{\ord}Q\) is ordered disjoint union:
every element of \(P\) is placed below every element of \(Q\).

Morphisms in \(\Met\) are monotone maps. The intended readings are:
\[
\begin{array}{cc}
\textbf{metadata construction} & \textbf{tensor reading}\\
\hline
[n] & \text{ordinary dense axis of length }n\\
P\times Q & \text{rectangular product of axes}\\
P+Q & \text{unordered collection of regions}\\
P+_{\ord}Q & \text{contiguous concatenation of regions}
\end{array}
\]

The subcollection generated only by finite total orders and cartesian products is the
dense rectangular fragment. The additional coproduct-like operations are used only to
remember how a packed dense surface decomposes into logical pieces.

\subsection{Displays}

A metadata display is a monotone map
\[
  p:X\to B
\]
in \(\Met\). We call \(X\) the dense surface and \(B\) the block poset. The fibre
\(p^{-1}(b)\) is the logical piece of the packed object lying over \(b\).

\begin{example}[Dense axis]
The ordinary dense axis of length \(n\), with no nontrivial packing metadata, may be
represented either as the discrete display
\[
  \id_{[n]}:[n]\to[n],
\]
where every cell is its own block, or as the one-block display
\[
  !:[n]\to[1],
\]
where all cells form a single dense block. These two embeddings have different
universal properties below.
\end{example}

\begin{example}[Packed ragged batch]
Let
\[
  X=[2]+_{\ord}([3]\times[4])+_{\ord}[5].
\]
The map
\[
  p:X\to[3]
\]
sending each ordered summand to its corresponding element of \([3]\) represents a
packed ragged batch with three logical fibres of shapes \(2\), \(3\times4\), and \(5\).
The underlying dense buffer has cardinality
\[
  |X|=2+12+5=19,
\]
but the display remembers the three-block decomposition.
\end{example}

\subsection{The category \texorpdfstring{\(\Disp(\Met)\)}{Disp(Met)}}

The category \(\Disp(\Met)\) has displays \(p:X\to B\) as objects. A morphism
\[
  (a,b):(p:X\to B)\longrightarrow(q:Y\to D)
\]
is a commutative square in \(\Met\):
\[
\begin{tikzcd}
X \arrow[r,"a"] \arrow[d,"p"']
  & Y \arrow[d,"q"] \\
B \arrow[r,"b"']
  & D .
\end{tikzcd}
\]
Equivalently,
\[
  q\circ a=b\circ p.
\]
The map \(a\) sends dense cells to dense cells, while \(b\) sends source blocks to
target blocks. The commutativity condition says that cell movement respects the
recorded packing metadata. Where $b$ is a bijection, the logical blocks are conserved up to reordering. More generally, we allow the metadata of logical blocks to be rewritten.

Composition is by pasting commutative squares. Identities are identity squares. Hence
\(\Disp(\Met)\) is the arrow category of \(\Met\), restricted to the display objects
under consideration.

\begin{remark}[Relation to partial reads]
The definition above records packing metadata, not padding values. Partial reads and
padding fillers are additional structure used by the operational semantics. They may
be added by replacing the dense-cell map \(a\) with a partial map, or by passing to a
partial-map enrichment. The present appendix only constructs the metadata grading
category and its relation to ordinary finite-set grading.
\end{remark}

\subsection{Symmetric monoidal structure}

The monoidal product of displays is given componentwise:
\[
  (p:X\to B)\boxtimes(q:Y\to D)
  :=
  (p\times q:X\times Y\to B\times D).
\]
On morphisms,
\[
  (a,b)\boxtimes(c,d):=(a\times c,\; b\times d).
\]
The monoidal unit is
\[
  \mathbb{1}:=\id_{[1]}:[1]\to[1].
\]

\begin{proposition}
\(\Disp(\Met)\), equipped with \(\boxtimes\), is symmetric monoidal.
\end{proposition}

\begin{proof}
Since \(\Met\) has finite products, the cartesian product supplies associators,
unitors, and symmetries:
\[
  (X\times Y)\times Z \cong X\times(Y\times Z),
  \qquad
  X\times[1]\cong X,
  \qquad
  X\times Y\cong Y\times X.
\]
Applying these isomorphisms simultaneously on dense surfaces and block posets gives
the corresponding structure maps in \(\Disp(\Met)\). For example, the symmetry
\[
  (p:X\to B)\boxtimes(q:Y\to D)
  \longrightarrow
  (q:Y\to D)\boxtimes(p:X\to B)
\]
is the square
\[
\begin{tikzcd}
X\times Y \arrow[r,"\sigma_{X,Y}"] \arrow[d,"p\times q"']
  & Y\times X \arrow[d,"q\times p"] \\
B\times D \arrow[r,"\sigma_{B,D}"']
  & D\times B .
\end{tikzcd}
\]
The square commutes by naturality of the cartesian symmetry. The associator and unitors
are identical componentwise arguments. The pentagon, triangle, and hexagon coherence
conditions hold because they hold in \(\Met\) for cartesian product, and
\(\Disp(\Met)\) inherits them levelwise.
\end{proof}

Thus \(\Disp(\Met)\) is eligible to grade the tube calculus: tensoring grades corresponds
to taking product structure both on the dense cells and on their displayed metadata.

\subsection{Forgetting and freely adding display structure}

We may characterise the relationship between the plain $\FinSet^\op$ grading and the metadata variant by a pair of adjunctions.

Let
\[
  U:\Disp(\Met)\to\FinSet
\]
be the forgetful functor sending a display \(p:X\to B\) to the underlying finite set
\(|X|\) of dense cells. On a morphism
\[
  (a,b):(p:X\to B)\to(q:Y\to D),
\]
\(U\) sends \((a,b)\) to the underlying finite function
\[
  |a|:|X|\to|Y|.
\]

There are two canonical ways to regard a finite set as a display.

First, the discrete display functor
\[
  \Delta:\FinSet\to\Disp(\Met)
\]
sends a finite set \(S\) to
\[
  \Delta S := \id_S:S\to S.
\]
This equips \(S\) with the finest possible trivial metadata: every cell is its own
block.

Second, the codiscrete or one-block display functor
\[
  \nabla:\FinSet\to\Disp(\Met)
\]
sends \(S\) to
\[
  \nabla S := !:S\to[1].
\]
This equips \(S\) with the coarsest possible trivial metadata: all cells lie in a single
block.

\begin{proposition}
There are adjunctions
\[
  \Delta \dashv U \dashv \nabla .
\]
\end{proposition}

\begin{proof}
For the left adjunction, let \(S\in\FinSet\) and let \(p:X\to B\) be a display. A morphism
\[
  \Delta S=(\id_S:S\to S)\longrightarrow(p:X\to B)
\]
is a commutative square
\[
\begin{tikzcd}
S \arrow[r,"a"] \arrow[d,"\id_S"']
  & X \arrow[d,"p"] \\
S \arrow[r,"b"']
  & B .
\end{tikzcd}
\]
The commutativity condition is \(p\circ a=b\). Hence \(b\) is uniquely determined by
\(a\). Therefore
\[
  \Disp(\Met)(\Delta S,p)\cong \FinSet(S,U p),
\]
naturally in \(S\) and \(p\). Thus \(\Delta\dashv U\).

For the right adjunction, a morphism
\[
  (p:X\to B)\longrightarrow \nabla S=(!:S\to[1])
\]
is a commutative square
\[
\begin{tikzcd}
X \arrow[r,"a"] \arrow[d,"p"']
  & S \arrow[d,"!"] \\
B \arrow[r,"!"']
  & [1] .
\end{tikzcd}
\]
The bottom map is unique, and the square commutes automatically. Hence such morphisms
are exactly finite functions \(a:X\to S\), giving a natural bijection
\[
  \Disp(\Met)(p,\nabla S)\cong \FinSet(U p,S).
\]
Thus \(U\dashv\nabla\).
\end{proof}

\section{Mask-augment duality: helper lemmas and detailed proof}
\label{app:maskprepdetail}

This appendix records the four helper lemmas that license individual transitions in the eleven-frame chain proving \Cref{thm:duality}, and reproduces the chain itself with each step's licence stated alongside it. The lemma proofs are routine pointwise verifications and are deferred for space.

\begin{lemma}[Augment $\Leftrightarrow$ pad/unpad]
\label{lem:augpad}
There exists a \texttt{pad}/\texttt{unpad} pair of functions (varying together in the padding parameter) such that the augmented K/V tube $\mathcal{A}_M(\cdot)$ equals a dense \texttt{pad} followed by the corresponding \texttt{unpad}.
\end{lemma}

\paragraph{The chain in detail.}
The chain reproduces \Cref{sec:maskprep}, displayed full-size; each equality is licensed by the lemma or move discussed inline.

\[ \tikzfig{maskprep/augment0} = \tikzfig{maskprep/augment1} \]
The first step follows definitionally from the padding.
\[ = \tikzfig{maskprep/augment2} = \tikzfig{maskprep/augment3} \]
These two steps follow from diagram-manipulation. Copy-maps nested inside copants allow copying of indexing-morphisms, and since tubings are string-diagrams in $\FinSet^\op$ (which has copy and delete), we may use counitality for the second equation.
\[ = \tikzfig{maskprep/augment4} \]
This step introducing a $0$-padding comes from the following lemma, from the middle frame to the right.
\begin{lemma}
\label{lem:FR2}
\[\tikzfig{lemmas/FR2}\]
\begin{proof}
The first equality results from straightforward application of Lemma \ref{lem:FR}. The second equality may be proven by observing that the initial socks in both diagram create \textcolor{Blue}{$m$} and \textcolor{Orange}{$n$} copies of the same input data, and the \texttt{unpad} in the second diagram selects a subset of size \textcolor{Blue}{$m$}, from copies, guaranteeing that the resulting matrix contains the same entries.
\end{proof}
\end{lemma}

\[ = \tikzfig{maskprep/augment5}\]
This equation comes from the following interaction between padding and products.

\begin{lemma}
\label{lem:pants1}
\[\tikzfig{lemmas/pants1}\]
\begin{proof}
The pants match padded indexes in both input tensors, feeding these in to $f$, producing $f(u, v)$.
\end{proof}
\end{lemma}

\[= \tikzfig{maskprep/augment6}\]

This equation (introducing a new pad-unpad pair) arises by the following definitional property of pad-unpad.

\begin{lemma}
\label{lem:FR}
\[\tikzfig{lemmas/FR}\]
\begin{proof}
See the definition of the \texttt{pad} and \texttt{unpad}, and their interaction defined in \Cref{prop:packpad}.
\end{proof}
\end{lemma}

\[ = \tikzfig{maskprep/augment7} \]

This step relies on padding-transformation properties of softmax:

\begin{lemma}
\label{lem:FRsmax}
\[\tikzfig{lemmas/FRsmax}\]
\begin{proof}
This lemma falls out of the fact that $e^{-\infty} = 0$, and padding an array with $-\infty$ preserve the softmax output of the existing array, sending the padded elements to 0.
\begin{align*}
  \sigma(\vec{x} || -\infty \cdots -\infty) = \sigma(\vec{x}) || 0 \cdots 0.
\end{align*}

\end{proof}
\end{lemma}

\[ = \tikzfig{maskprep/augment8} \]

This step relies on addition being invariant to 0-padding.
\begin{lemma}
\label{lem:maskinvadd}
\[\tikzfig{lemmas/maskinvadd}\]
\begin{proof}
The sum of an array is invariant under padding with $0$s.
\begin{align*}
  \sum \vec{x}|| 0 \cdots 0 = \sum \vec{x}
\end{align*}
\end{proof}
\end{lemma}

\[ = \tikzfig{maskprep/augment9} \]

This step relies on an interaction law between paddings and general binary operations across tensors.

\begin{lemma}
\label{lem:RFmult}
\[\tikzfig{lemmas/RFmult}\]
\begin{proof}
Using \Cref{lem:pants1}, it is possible to push the pad on the output backwards through the product, producing a $\texttt{pad}_0$ on one input, and an arbitrary \texttt{pad} on the other. The pad value here maybe left unspecified, since $\forall u. 0 \cdot u = 0$.
\end{proof}
\end{lemma}

\[ = \tikzfig{maskprep/augment10}. \]

The final step is another interaction law of softmax with padding.

\begin{lemma}
\label{lem:FRsmax2}
\[\tikzfig{lemmas/FRsmax2}\]
\begin{proof}
  The equivalence is a consequence of the following known property of the softmax function.
\begin{align*}
  \sigma(\vec{x} || -\infty \cdots -\infty) = \sigma(\vec{x}) || 0 \cdots 0.
\end{align*}
Inputs to softmax may be padded with $-\infty$ leaving the output unchanged, and sending the padding to 0. This is a standard trick used to implement masked self-attention.
\end{proof}
\end{lemma}

\clearpage

\section{Mask-augment duality: Comparison with code}
\label{app:maskprepdiagrams}

Below we repeat the eleven frames of the derivation, each paired with the corresponding \texttt{forward} function, implemented in \texttt{torch + einops}. The code transcription proves the same identity without diagrams. Its length is the point: the graphical proof is a typed compression of a structural tensor-program derivation, not an informal illustration of it.

\noindent\begin{tabular}{@{}>{\centering\arraybackslash}p{0.42\textwidth}|>{\centering\arraybackslash}p{0.55\textwidth}@{}}
\textbf{Diagram} & \textbf{torch + einops} \\
\end{tabular}
\par\vspace{0.4em}

\begin{lrbox}{\stepcodebox}\begin{minipage}{0.55\textwidth}
\begin{lstlisting}
# Step 0, augmented version of attention
def forward(self, x):
    x_q, bx_kv = x, self.augment(x)
    q = self.w_q(x_q)
    k = self.w_k(bx_kv)
    v = self.w_v(bx_kv)
    scores = einsum(q, k, 'i d, i j d -> i j') * self.scale
    attn = scores.softmax(-1)
    out = einsum(attn, v, 'i j, i j d -> i d')
    return self.proj(out)
\end{lstlisting}
\end{minipage}\end{lrbox}%
\noindent\begin{tabular}{@{}>{\centering\arraybackslash}p{0.42\textwidth}|p{0.55\textwidth}@{}}
\resizebox{\linewidth}{!}{\tikzfig{maskprep/augment0}} & \usebox{\stepcodebox} \\
\end{tabular}
\par\vspace{1.2em}

\begin{lrbox}{\stepcodebox}\begin{minipage}{0.55\textwidth}
\begin{lstlisting}
# Step 1, prep := unpad . repeat
def forward(self, x):
    x_q, bx_kv = x, self.unpad(repeat(x, 'j d -> i j d', i=x.shape[0]), -INF)
    q = self.w_q(x_q)
    k = self.w_k(bx_kv)
    v = self.w_v(bx_kv)
    scores = einsum(q, k, 'i d, i j d -> i j') * self.scale
    attn = scores.softmax(-1)
    out = einsum(attn, v, 'i j, i j d -> i d')
    return self.proj(out)
\end{lstlisting}
\end{minipage}\end{lrbox}%
\noindent\begin{tabular}{@{}>{\centering\arraybackslash}p{0.42\textwidth}|p{0.55\textwidth}@{}}
\resizebox{\linewidth}{!}{\tikzfig{maskprep/augment1}} & \usebox{\stepcodebox} \\
\end{tabular}
\par\vspace{1.2em}

\begin{lrbox}{\stepcodebox}\begin{minipage}{0.55\textwidth}
\begin{lstlisting}
# Step 2, pushing unpad.repeat through copy
def forward(self, x):
    x_q = x
    bx_k = self.unpad(repeat(x, 'j d -> i j d', i=x.shape[0]), -INF)
    bx_v = self.unpad(repeat(x, 'j d -> i j d', i=x.shape[0]), -INF)
    q = self.w_q(x_q)
    k = self.w_k(bx_k)
    v = self.w_v(bx_v)
    scores = einsum(q, k, 'i d, i j d -> i j') * self.scale
    attn = scores.softmax(-1)
    out = einsum(attn, v, 'i j, i j d -> i d')
    return self.proj(out)
\end{lstlisting}
\end{minipage}\end{lrbox}%
\noindent\begin{tabular}{@{}>{\centering\arraybackslash}p{0.42\textwidth}|p{0.55\textwidth}@{}}
\resizebox{\linewidth}{!}{\tikzfig{maskprep/augment2}} & \usebox{\stepcodebox} \\
\end{tabular}
\par\vspace{1.2em}

\begin{lrbox}{\stepcodebox}\begin{minipage}{0.55\textwidth}
\begin{lstlisting}
# Step 3, introduce new axis on q
def forward(self, x):
    x_q, x_v = x, x
    bx_k = self.unpad(repeat(x, 'j d -> i j d', i=x.shape[0]), -INF)
    q = repeat(x_q, 'i d -> i j d', j=bx_k.shape[1])
    q = self.w_q(q)
    k = self.w_k(bx_k)
    v = self.w_v(x_v)
    bx_v = self.unpad(repeat(v, 'j d -> i j d', i=v.shape[0]), -INF)
    scores = einsum(q, k, 'i j d, i j d -> i j') * self.scale
    attn = scores.softmax(-1)
    out = einsum(attn, bx_v, 'i j, i j d -> i d')
    return self.proj(out)
\end{lstlisting}
\end{minipage}\end{lrbox}%
\noindent\begin{tabular}{@{}>{\centering\arraybackslash}p{0.42\textwidth}|p{0.55\textwidth}@{}}
\resizebox{\linewidth}{!}{\tikzfig{maskprep/augment3}} & \usebox{\stepcodebox} \\
\end{tabular}
\par\vspace{1.2em}

\begin{lrbox}{\stepcodebox}\begin{minipage}{0.55\textwidth}
\begin{lstlisting}
# Step 4, unpad q
def forward(self, x):
    x_q, x_v = x, x
    bx_k = self.unpad(repeat(x, 'j d -> i j d', i=x.shape[0]), -INF)
    q = self.unpad(repeat(x_q, 'i d -> i j d', j=x.shape[0]), 1)
    q = self.w_q(q)
    k = self.w_k(bx_k)
    v = self.w_v(x_v)
    bx_v = self.unpad(repeat(v, 'j d -> i j d', i=v.shape[0]), -INF)
    scores = einsum(q, k, 'i j d, i j d -> i j') * self.scale
    attn = scores.softmax(-1)
    out = einsum(attn, bx_v, 'i j, i j d -> i d')
    return self.proj(out)
\end{lstlisting}
\end{minipage}\end{lrbox}%
\noindent\begin{tabular}{@{}>{\centering\arraybackslash}p{0.42\textwidth}|p{0.55\textwidth}@{}}
\resizebox{\linewidth}{!}{\tikzfig{maskprep/augment4}} & \usebox{\stepcodebox} \\
\end{tabular}
\par\vspace{1.2em}

\begin{lrbox}{\stepcodebox}\begin{minipage}{0.55\textwidth}
\begin{lstlisting}
# Step 5, push unpad through the einsum
def forward(self, x):
    x_q, x_v = x, x
    bx_k = repeat(x, 'j d -> i j d', i=x.shape[0])
    q = repeat(x_q, 'i d -> i j d', j=x.shape[0])
    q = self.w_q(q)
    k = self.w_k(bx_k)
    v = self.w_v(x_v)
    bx_v = self.unpad(repeat(v, 'j d -> i j d', i=v.shape[0]), -INF)
    scores = einsum(q, k, 'i j d, i j d -> i j') * self.scale
    scores = self.unpad(scores, -INF)
    attn = scores.softmax(-1)
    out = einsum(attn, bx_v, 'i j, i j d -> i d')
    return self.proj(out)
\end{lstlisting}
\end{minipage}\end{lrbox}%
\noindent\begin{tabular}{@{}>{\centering\arraybackslash}p{0.42\textwidth}|p{0.55\textwidth}@{}}
\resizebox{\linewidth}{!}{\tikzfig{maskprep/augment5}} & \usebox{\stepcodebox} \\
\end{tabular}
\par\vspace{1.2em}

\begin{lrbox}{\stepcodebox}\begin{minipage}{0.55\textwidth}
\begin{lstlisting}
# Step 6, introduce pad then unpad identity
def forward(self, x):
    x_q, x_v = x, x
    bx_k = repeat(x, 'j d -> i j d', i=x.shape[0])
    q = repeat(x_q, 'i d -> i j d', j=x.shape[0])
    q = self.w_q(q)
    k = self.w_k(bx_k)
    v = self.w_v(x_v)
    bx_v = self.unpad(repeat(v, 'j d -> i j d', i=v.shape[0]), -INF)
    scores = einsum(q, k, 'i j d, i j d -> i j') * self.scale
    scores = self.unpad(scores, -INF)
    scores = self.pad(scores, -INF)
    scores = self.unpad(scores, -INF)
    attn = scores.softmax(-1)
    out = einsum(attn, bx_v, 'i j, i j d -> i d')
    return self.proj(out)
\end{lstlisting}
\end{minipage}\end{lrbox}%
\noindent\begin{tabular}{@{}>{\centering\arraybackslash}p{0.42\textwidth}|p{0.55\textwidth}@{}}
\resizebox{\linewidth}{!}{\tikzfig{maskprep/augment6}} & \usebox{\stepcodebox} \\
\end{tabular}
\par\vspace{1.2em}

\begin{lrbox}{\stepcodebox}\begin{minipage}{0.55\textwidth}
\begin{lstlisting}
# Step 7, push unpad through the softmax
def forward(self, x):
    x_q, x_v = x, x
    bx_k = repeat(x, 'j d -> i j d', i=x.shape[0])
    q = repeat(x_q, 'i d -> i j d', j=x.shape[0])
    q = self.w_q(q)
    k = self.w_k(bx_k)
    v = self.w_v(x_v)
    bx_v = self.unpad(repeat(v, 'j d -> i j d', i=v.shape[0]), -INF)
    scores = einsum(q, k, 'i j d, i j d -> i j') * self.scale
    scores = self.unpad(scores, -INF)
    scores = self.pad(scores, -INF)
    attn = scores.softmax(-1)
    attn = self.unpad(attn, 0)
    out = einsum(attn, bx_v, 'i j, i j d -> i d')
    return self.proj(out)
\end{lstlisting}
\end{minipage}\end{lrbox}%
\noindent\begin{tabular}{@{}>{\centering\arraybackslash}p{0.42\textwidth}|p{0.55\textwidth}@{}}
\resizebox{\linewidth}{!}{\tikzfig{maskprep/augment7}} & \usebox{\stepcodebox} \\
\end{tabular}
\par\vspace{1.2em}

\begin{lrbox}{\stepcodebox}\begin{minipage}{0.55\textwidth}
\begin{lstlisting}
# Step 8, pad before final sum
def forward(self, x):
    x_q, x_v = x, x
    bx_k = repeat(x, 'j d -> i j d', i=x.shape[0])
    q = repeat(x_q, 'i d -> i j d', j=x.shape[0])
    q = self.w_q(q)
    k = self.w_k(bx_k)
    v = self.w_v(x_v)
    bx_v = self.unpad(repeat(v, 'j d -> i j d', i=v.shape[0]), -INF)
    scores = einsum(q, k, 'i j d, i j d -> i j') * self.scale
    scores = self.unpad(scores, -INF)
    scores = self.pad(scores, -INF)
    attn = scores.softmax(-1)
    attn = self.unpad(attn, 0)
    out = einsum(attn, bx_v, 'i j, i j d -> i d j')
    out = self.pad(out, 0)
    out = out.sum(-1)
    return self.proj(out)
\end{lstlisting}
\end{minipage}\end{lrbox}%
\noindent\begin{tabular}{@{}>{\centering\arraybackslash}p{0.42\textwidth}|p{0.55\textwidth}@{}}
\resizebox{\linewidth}{!}{\tikzfig{maskprep/augment8}} & \usebox{\stepcodebox} \\
\end{tabular}
\par\vspace{1.2em}

\begin{lrbox}{\stepcodebox}\begin{minipage}{0.55\textwidth}
\begin{lstlisting}
# Step 9, eliminate v unpad and final pad
def forward(self, x):
    x_q, x_v = x, x
    bx_k = repeat(x, 'j d -> i j d', i=x.shape[0])
    q = repeat(x_q, 'i d -> i j d', j=x.shape[0])
    q = self.w_q(q)
    k = self.w_k(bx_k)
    v = self.w_v(x_v)
    bx_v = repeat(v, 'j d -> i j d', i=v.shape[0])
    scores = einsum(q, k, 'i j d, i j d -> i j') * self.scale
    scores = self.unpad(scores, -INF)
    scores = self.pad(scores, -INF)
    attn = scores.softmax(-1)
    attn = self.unpad(attn, 0)
    attn = self.pad(attn, 0)
    out = einsum(attn, bx_v, 'i j, i j d -> i d j')
    out = out.sum(-1)
    return self.proj(out)
\end{lstlisting}
\end{minipage}\end{lrbox}%
\noindent\begin{tabular}{@{}>{\centering\arraybackslash}p{0.42\textwidth}|p{0.55\textwidth}@{}}
\resizebox{\linewidth}{!}{\tikzfig{maskprep/augment9}} & \usebox{\stepcodebox} \\
\end{tabular}
\par\vspace{1.2em}

\begin{lrbox}{\stepcodebox}\begin{minipage}{0.55\textwidth}
\begin{lstlisting}
# Step 10, masked attention
def forward(self, x):
    x_q, x_v = x, x
    bx_k = repeat(x, 'j d -> i j d', i=x.shape[0])
    q = repeat(x_q, 'i d -> i j d', j=x.shape[0])
    q = self.w_q(q)
    k = self.w_k(bx_k)
    v = self.w_v(x_v)
    scores = einsum(q, k, 'i j d, i j d -> i j') * self.scale

    # MASK
    scores = self.unpad(scores, -INF)
    scores = self.pad(scores, -INF)

    attn = scores.softmax(-1)
    out = einsum(attn, v, 'i j, j d -> i d')
    return self.proj(out)
\end{lstlisting}
\end{minipage}\end{lrbox}%
\noindent\begin{tabular}{@{}>{\centering\arraybackslash}p{0.42\textwidth}|p{0.55\textwidth}@{}}
\resizebox{\linewidth}{!}{\tikzfig{maskprep/augment10}} & \usebox{\stepcodebox} \\
\end{tabular}
\par\vspace{1.2em}

\clearpage


\section{Three systems optimisations as corollaries}
\label{app:experiments}

The three corollaries below recover known systems optimisations as instances of the fibrewise decomposition in \Cref{thm:duality}. Prefix and DAG hoisting, packed-document compaction, and bounded-component scheduling come from disparate threads of the systems literature and were engineered independently; the duality places them on a common footing and yields a quantitative speedup prediction for each. We attribute the corollaries to their known instances and validate the predictions on a single Apple M4 Max (128\,GB unified memory). The following corollaries are not priority claims over the cited systems optimisations. They are validation targets: independently discovered efficient implementations should be recovered by any structural theory that correctly explains attention masks.

\subsection{Corollary 1: hoisting common fibres}
\label{app:exp:hoisting}

\paragraph{Known instances.}
For $n$ queries sharing a contiguous left-prefix, prompt-prefix caching computes the prefix's K/V once and reuses it: PagedAttention with prefix sharing in vLLM \citep{kwon2023vllm}, RadixAttention in SGLang \citep{zheng2024sglang}, and llama.cpp's prompt cache. For queries sharing an interior segment (tree-of-thoughts \citep{yao2023tree}, self-consistency sampling \citep{wang2023selfconsistency}, shared-context RAG), Hydragen \citep{juravsky2024hydragen} runs varying suffixes against the shared K/V; speculative-decoding tree caches do the same \citep{leviathan2023speculative}.

\paragraph{Validation.}
Hermes~4~70B (MLX 4-bit \citep{lmstudiocommunity}) via \texttt{mlx-lm} \citep{mlxlm}. Strict-prefix sweep: shared prompt of length $P$, $n$ candidate suffixes of length $s$, predicted speedup $n(P+s)^2 / (P^2 + 2nPs + ns^2)$ over the grid $P{\in}\{1024,2048\}$, $n{\in}\{2,4,8\}$; $\arg\max$-of-best-of-$n$ parity throughout, $6.82\times$ measured at $P{=}1024,\,s{=}64,\,n{=}8$ against $4.45\times$ predicted (\Cref{fig:summary}, left). DAG bench (root, $m$ intermediates, $k$ leaves; $R{=}256,\,I{=}32,\,L{=}16,\,m{=}4,\,k{=}4$): $1.20\times$ above strict-prefix, $4.08\times$ above naive.

\subsection{Corollary 2: compacting disjoint fibres into dense slabs}
\label{app:exp:compaction}

\paragraph{Known instances.}
Sequence packing concatenates short documents into a fixed-length training sequence with a block-diagonal mask \citep{krell2021packing,kosec2021packing}. The same pattern is the canonical use of FlashAttention's variable-length API \citep{dao2022flashattention,dao2023flashattention2}, xFormers' \texttt{BlockDiagonalMask} \citep{xformers}, and FlexAttention's \texttt{block\_diag} mask\_mod \citep{flexattention}.

\paragraph{Validation.}
Synthetic Q/K/V at $T\in\{4096,8192,16384\}$, document counts $n_d\in\{4,8,16\}$, average visible density $\rho = \sum_i \ell_i^2 / (\sum_i \ell_i)^2$. Predicted speedup $\approx 1/\rho$; best cell $9.25\times$ measured against $10.06\times$ predicted at $\rho\approx 0.099$, parity $\max|\Delta|\le 4\times 10^{-3}$ at fp16 (\Cref{fig:summary}, middle). A whole-model patch on Hermes~4~70B's \texttt{mlx\_lm.models.llama} attention (\texttt{exp02\_compaction/hermes\_mlx\_patch.py}) reproduces the pattern with bit-exact parity on a single full-causal segment.

\subsection{Corollary 3: bounded-component masks compile to row-state machines}
\label{app:exp:scheduling}

\paragraph{Known instances.}
Sliding-window attention restricts each query to a bounded neighbourhood: Sparse~Transformer \citep{child2019sparsetransformer}, Longformer \citep{beltagy2020longformer}, BigBird \citep{zaheer2020bigbird}, Mistral~7B \citep{jiang2023mistral}, alternating local-global layers in Gemma~3/4 \citep{gemma2024,lmstudiocommunity}, GPT-OSS~120B \citep{lmstudiocommunity}, Streaming-LLM \citep{xiao2024streamingllm}. Kernel-side, FlashAttention-2 supports custom masks \citep{dao2023flashattention2}; FlexAttention compiles \texttt{score\_mod}/\texttt{mask\_mod} into a \texttt{BlockMask} \citep{flexattention}; FlashMask records each row's visible columns as an interval list \citep{wang2024flashmask}.

\paragraph{Validation.}
Two paths share one \texttt{compile\_program}: measured on synthetic SWA at varying $T$ and window $w$; predicted on four GGUF frontier models (Gemma~4 31B/26B-A4B, Llama~3.3 70B, GPT-OSS 120B) by reading sliding-window metadata from each GGUF header. Measured tracks $T/(c\,w)$ over $T/w\in[4,128]$ at $73$--$95\%$ of theoretical maximum; the gap is host-side gather/scatter. At $T{=}8192$: $32\times$ per sliding layer for GPT-OSS 120B (window 128, 18/36 sliding); $7.1\times$ for Gemma~4 31B (window 1024, 50/60 sliding). Whole-model effective speedup scales by sliding-layer fraction (\Cref{fig:summary}, right).

\begin{figure}[!ht]
\centering
\includegraphics[width=\linewidth]{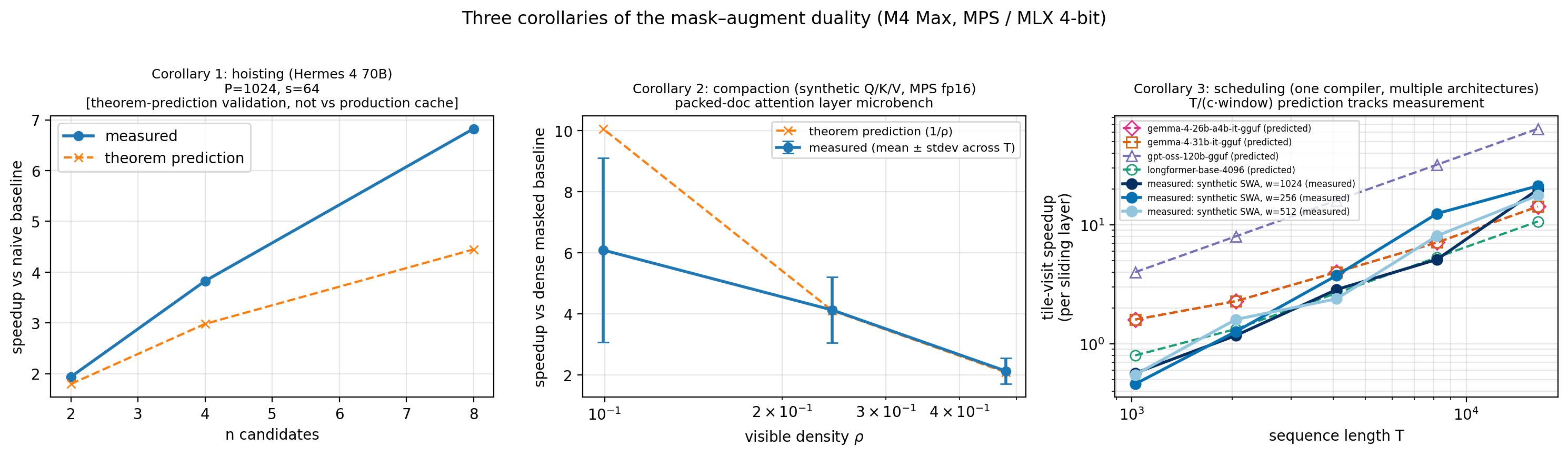}
\caption[Three corollaries of the mask-augment duality]{\textbf{Three corollaries of the mask-augment duality.} Left: Corollary~1, strict-prefix hoisting on Hermes~4~70B. Middle: Corollary~2, packed-document compaction on synthetic Q/K/V at fp16 (mean $\pm$ stdev across three $T$ values per $n_d$). Right: Corollary~3, bounded-component scheduling: predicted (open markers, dashed) on four GGUF frontier models from architectural metadata, measured (filled, solid) on synthetic SWA at three window sizes.}
\label{fig:summary}
\end{figure}

\subsection{Methodology and limitations}
\label{app:exp:methodology}

\begin{itemize}
\item Corollary 1's strict-prefix sweep is theorem-prediction validation, not a competitive benchmark; production stacks \citep{kwon2023vllm,zheng2024sglang} already implement strict-prefix caching. The genuine gap is the DAG case.
\item Corollary 3's measured anchor is synthetic SWA. Local Mistral-7B-v0.1 and Longformer weights were absent or partially incompatible (Longformer's split global-token projections break a unified-path patch); the predicted-side reads published architectural metadata.
\item Hermes~4~70B's parity drift is 4-bit-bounded ($\sim 0.5$--$1.0$ nats per cell on log-prob sums of 64--256 tokens). The $\arg\max$-of-best-of-$n$ parity claim holds throughout; per-token claims would tighten to $\sim 10^{-2}$ nats at fp16.
\item No comparison against custom block-sparse kernels on non-Apple silicon. Lowering compiled displays to FlexAttention or FlashMask on CUDA is the natural follow-up; the harness emits the abstract \texttt{Segment}/\texttt{RowState} objects such a lowering would consume.
\end{itemize}

\end{document}